%% file: main.tex
\documentclass{article}

\usepackage{PRIMEarxiv}
\usepackage{algorithm}
\usepackage{amsmath}
\usepackage{hhline}
\usepackage{graphicx}
\usepackage{subcaption}

\usepackage[utf8]{inputenc} 
\usepackage[T1]{fontenc}    
\usepackage{hyperref}       
\usepackage{url}            
\usepackage{booktabs}       
\usepackage{amsfonts}       
\usepackage{nicefrac}       
\usepackage{microtype}      
\usepackage{xcolor}         
\usepackage{algorithm,algpseudocode}
\usepackage{amsthm}
\usepackage{enumitem}

\DeclareMathOperator*{\argmin}{argmin}

\newcommand{\Sp}[1]{\left(#1\right)}
\newcommand{\Mp}[1]{\left[#1\right]}
\newcommand{\Bp}[1]{\left\{#1\right\}}

\newcommand{\Norm}[1]{\left\|#1\right\|}

\newcommand{\inner}[1]{\left\langle#1\right\rangle}

\newcommand{\A}{\mathcal{A}}
\newcommand{\B}{\mathbb{B}}
\newcommand{\D}{\mathcal{D}}
\newcommand{\E}{\mathbb{E}}

\renewcommand{\P}{\mathbb{P}}
\renewcommand{\S}{\mathcal{S}}

\renewcommand{\a}{\mathbf{a}}
\newcommand{\R}{\mathbb{R}}

\renewcommand{\r}{\mathbf{r}}

\newcommand{\M}{\mathcal{M}}

\newcommand{\RNum}[1]{\uppercase\expandafter{\romannumeral #1\relax}}

\newtheorem{assumption}{Assumption}
\newtheorem{theorem}{Theorem}
\newtheorem{definition}{Definition}
\newtheorem{lemma}{Lemma}

\usepackage{natbib}

\title{
Preference-Based Multi-Agent Reinforcement Learning: Data Coverage and Algorithmic Techniques
}

\author{
  Natalia Zhang\thanks{
  Tsinghua University, \texttt{zsxn21@mails.tinghua.edu.cn}.
  University of Washington, \texttt{wxqkaxdd@uw.edu}.
  University of Washington, \texttt{qwcui@cs.washington.edu}.
  These authors contributed equally to this work. The work was done when Natalia Zhang was visiting the University of Washington.
  }
  \And Xinqi Wang\footnotemark[1]
  \And Qiwen Cui\footnotemark[1]
  \And Runlong Zhou\thanks{University of Washington, \texttt{vectorzh@cs.washington.edu}.}
  \And Sham M.~Kakade\thanks{Harvard University, \texttt{sham@seas.harvard.edu}.}
  \And Simon S.~Du\thanks{University of Washington, \texttt{ssdu@cs.washington.edu}.}
}
\begin{document}
\maketitle

\input{text/Abstract}

\input{text/Introduction}

\input{text/RelatedWorks}
\input{text/Preliminaries}

\input{text/DatasetTheory}
\input{text/Method}
\input{text/Experimental}

\input{text/Discussions}

\bibliographystyle{plainnat}
\bibliography{ref.bib}

\newpage
\appendix

\input{text/Appendix_Proof}

\input{text/Appendix_Experiments}

\end{document}

%% file: text/Abstract.tex
\begin{abstract}
We initiate the study of Preference-Based Multi-Agent Reinforcement Learning (PbMARL),
exploring both theoretical foundations and empirical validations. We define the task as identifying the Nash equilibrium from a preference-only offline dataset in general-sum games, a problem marked by the challenge of sparse feedback signals. Our theory establishes the upper complexity bounds for Nash Equilibrium in effective PbMARL, demonstrating that single-policy coverage is inadequate and highlighting the importance of unilateral dataset coverage. These theoretical insights are verified through comprehensive experiments.
To enhance the practical performance, we further introduce two algorithmic techniques. 
(1) We propose a Mean Squared Error (MSE) regularization along the time axis to achieve a more uniform reward distribution and improve reward learning outcomes. (2) We propose an additional penalty based on the distribution of the dataset to incorporate pessimism, improving stability and effectiveness during training.
Our findings underscore the multifaceted approach required for PbMARL, paving the way for effective preference-based multi-agent systems.
\end{abstract}

\keywords{multi-agent reinforcement learning \and reinforcement learning from human feedback \and dataset coverage}

%% file: text/Introduction.tex
\section{Introduction}

 Large language models (LLMs) have achieved significant progress in natural language interaction, knowledge acquisition, instruction following, planning and reasoning, which has been recognized as the sparks for AGI \citep{bubeck2023sparks}. The evolution of LLMs fosters the field of agent systems, wherein LLMs act as the central intelligence \citep{xi2023rise}. In these systems, multiple LLMs can interact with each other as well as with external tools. For instance, MetaGPT assigns LLM agents various roles, akin to those in a technology company, enabling them to cooperate on complex software engineering tasks \citep{hong2023metagpt}. 
 
 Despite some empirical successes in agent systems utilizing closed-source LLMs, finetuning these systems and aligning them with human preferences remains a challenge. Reinforcement learning from human feedback (RLHF) has played an important role in aligning LLMs with human preferences \citep{christiano2017deep,ziegler2019fine}. However, unexpected behavior can arise when multiple LLMs interact with each other. In addition, reward design has been a hard problem in multi-agent reinforcement learning \citep{devlin2011empirical}. Thus, it is crucial to further align the multi-agent system from preference feedback.

We address this problem through both theoretical analysis and empirical experiments. 
Theoretically, we characterize the dataset coverage condition for PbMARL that enables learning the Nash equilibrium, which serves as a favorable policy for each player. 
Empirically, we validate our theoretical insights through comprehensive experiments utilizing the proposed algorithmic techniques.

\subsection{Contributions and Technical Novelties}

\paragraph{1. Necessary and Sufficient Dataset Coverage Condition for PbMARL.} In single-agent RLHF, \citep{zhu2023principled} demonstrated that single policy coverage is sufficient for learning the optimal policy. However, we prove that this condition no longer holds for PbMARL by providing a counterexample. Instead, we introduce an algorithm that operates under unilateral coverage, a condition derived from offline MARL \citep{cui2022offline, zhong2022pessimistic}. Specifically, this condition requires the dataset to cover all unilateral deviations from a Nash equilibrium policy. For further details, see Section~\ref{section:theory}.

\begin{figure}
    \centering
    \includegraphics[width=.5\textwidth]{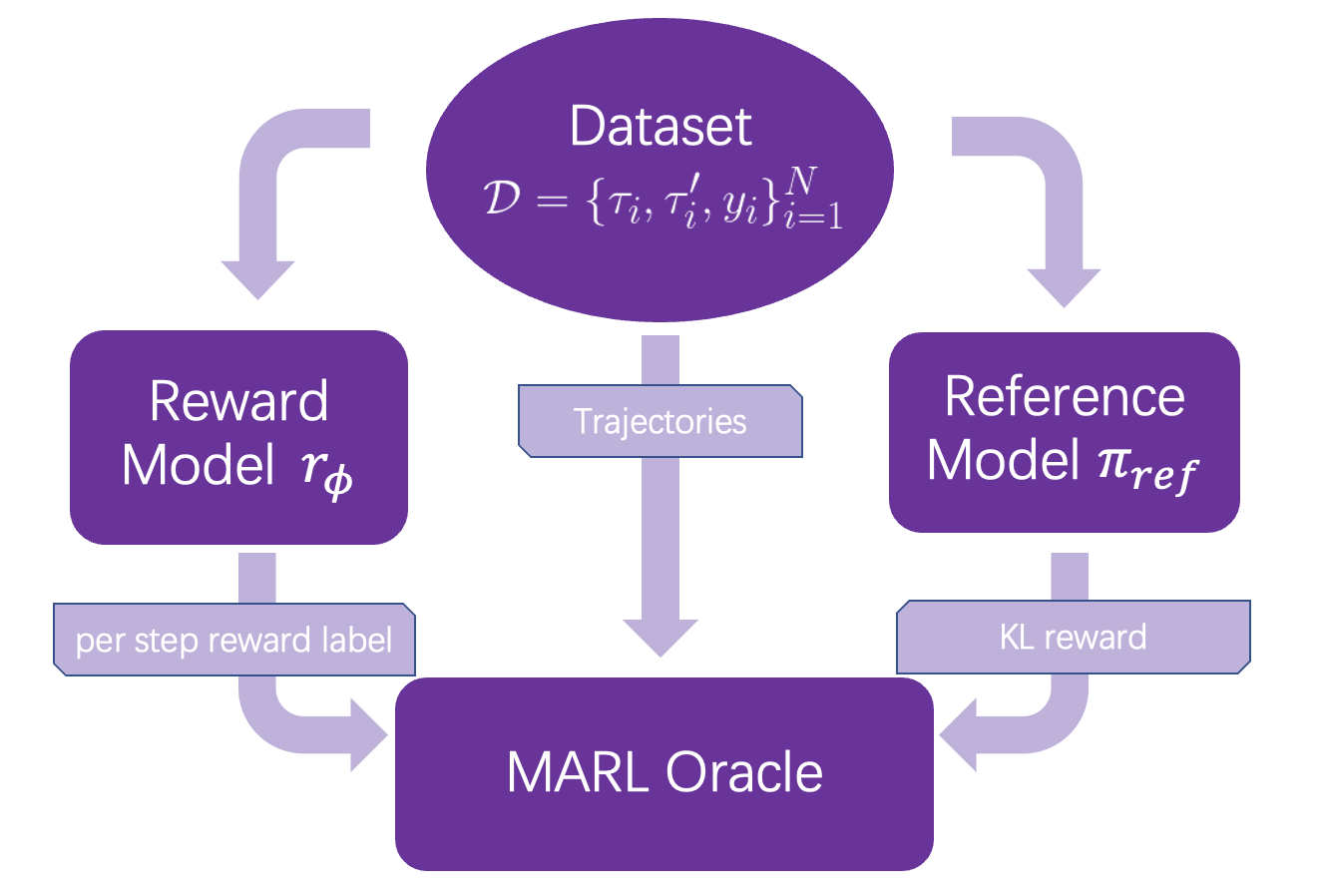}
    \caption{The overall pipeline of offline PbMARL. $\mathcal{D}$ is the preference dataset where $\tau_i, \tau_i'$ are trajectories and $\mathbf{y}_i \in \{1,-1\}^m$ indicates which trajectory is preferred by each agent. $r_\phi$ is the learned reward. $\pi_{b}$ is the learned reference policy using imitation learning.}
    \label{fig: pipline}
\end{figure}

\paragraph{2. Algorithmic Techniques for Practical Performance.}
As a foundational exploration into PbMARL research, we focus on employing the simplest learning framework, incorporating only the essential techniques necessary to ensure the approach's feasibility. The framework consists of three key components: 1) leveraging the preference dataset to learn a reward function, 2) mitigating extrapolation errors with pessimism, and 3) determining the final policy. Figure~\ref{fig: pipline} provides an overview of the process.

However, additional algorithmic techniques are required to identify a robust policy, even when the dataset demonstrates good coverage according to our theoretical insights.

\begin{itemize} 
    \item \textbf{Reward regularization.}  
    We observed that the reward learned through standard Maximum Likelihood Estimation (MLE) is sparse and spiky, making it difficult for standard RL algorithms to utilize effectively (cf. Figure \ref{fig:rewardlines} (b2)). 
    To address this, we introduce an additional Mean Squared Error (MSE) loss between the predictions of adjacent time steps as a form of regularization. This regularization helps to prevent the model from accumulating reward signals solely at the final time step or relying on reward-irrelevant observation patterns, which could otherwise result in the complete failure in producing meaningful predictions.
    

    \item \textbf{Dataset Distribution-Based Pessimism.}
    To mitigate the extrapolation error in offline RL, we add an extra reward term based on the density of a certain state-action pair in the dataset to implement pessimism. In our approach, an imitation learning agent is trained to model the density function.
    The final policy is then trained using a DQN-based Value Decomposition Network (VDN) \citep{mnih2013playing, sunehag2017valuedecomposition}. 
    Our ablation study demonstrates the critical role of appropriately tuning the reward coefficient to ensure training stability and performance (see Table \ref{tab: tech-mpe}).
\end{itemize}

\paragraph{3. Experiment Results.}
Our experiments, following the pipeline described above, confirm the theoretical necessity of unilateral coverage. 
We performed extensive ablation studies across three Multi-Agent Particle Environment (MPE) scenarios—Spread-v3, Tag-v3, and Reference-v3 \citep{mordatch2017emergence}—as well as the popular Overcooked environment \citep{carroll2020utilitylearninghumanshumanai}. These studies focused on the hyperparameter selection for the reward regularization coefficient $\alpha$, pessimism coefficient $\beta$, and dataset diversity. 
The empirical results (Table \ref{tab:exp-mpe}) demonstrate that: 1) augmenting expert demonstrations with trivial trajectories significantly improves performance, 2) unilateral datasets are advantageous, and 3) dataset diversity contributes to lower variance.
Our ablation experiments underscore the effectiveness of the proposed algorithmic techniques.
Additionally, we introduced a principled standardization technique that efficiently tunes hyperparameters across all environments and datasets.


%% file: text/RelatedWorks.tex
\section{Related Works}

\paragraph{Reinforcement Learning from Human Feedback (RLHF).}

RLHF, or preference-based RL (PbRL), plays a pivotal role in alignment with various tasks such as video games \citep{warnell2018deep,brown2019extrapolating}, robotics \citep{jain2013learning,kupcsik2016learning,christiano2023deep,shin2023benchmarks}, image augmentation \citep{metcalf2024sampleefficient}, and large language models \citep{ziegler2020finetuning,wu2021recursively,nakano2022webgpt,menick2022teaching,stiennon2022learning,bai2022training,glaese2022improving,ganguli2022red,ouyang2022training}.
Additionally, a body of work focuses on the reward models behind preference data \citep{Sadigh2017ActivePL,bıyık2018batch,gao2022scaling,hejna2023inverse}.
Recent works like VIPO \citep{cen2024valueincentivizedpreferenceoptimizationunified} incorporates uncertainty-aware regularization into the reward model, while \citep{liu2024provablymitigatingoveroptimizationrlhf} address over-optimization using adversarial regularization.
Direct preference optimization (DPO, \cite{rafailov2023direct}) and its variants \citep{azar2023general,rafailov2024r} approach RLHF without directly handling the reward model.
Theoretical studies have also explored guarantees, such as sample complexity and regret, and the limitations of certain RLHF algorithms \citep{novoseller2020dueling,xu2020preferencebased,pacchiano2023dueling,chen2022humanintheloop,razin2023vanishing,zhu2024principled,wang2023rlhf,xiong2024iterative,zhu2024iterative}.

\paragraph{Offline Reinforcement Learning.}
Offline RL \citep{Lange2012,levine2020offline} has achieved success in a wide range of real-world applications, including robotics \citep{pinto2015supersizing,levine2016learning,chebotar2021actionable,kumar2023pretraining}, healthcare \citep{raghu2017deep,wang2018supervised}, and autonomous driving \citep{shi2021offline,lee2024ad4rl}.
Key algorithms such as Behavior Cloning, BRAC \citep{wu2019behavior}, BEAR \citep{kumar2019stabilizing}, and CQL \citep{kumar2020conservative,lyu2024mildly} have driven these successes.
Theoretical research on offline RL has primarily focused on sample complexity under various dataset coverage assumptions \cite{le2019batch,chen2019informationtheoretic,yin2020nearoptimal,rashidinejad2023bridging,yin2021nearoptimal,yin2022nearoptimal,shi2022pessimistic,nguyentang2022sample,xie2022policy,xiong2023nearly,li2024settling,xie2023bellmanconsistentpessimismofflinereinforcement,mete2021rewardbiasedmaximumlikelihood}.

\paragraph{Multi-Agent Reinforcement Learning (MARL).}
Many real-world scenarios are naturally modeled as multi-agent environments, whether cooperative or competitive. 
As a result, MARL has gained popularity in video games  \citep{tian2017elf,vinyals2017starcraft,Silver2017MasteringTG,Vinyals2019GrandmasterLI}, network design \citep{shamsoshoara2018distributed,9109331}, energy sharing \citep{Prasad2018MultiagentDR}, and autonomous driving \citep{palanisamy2019multiagent,8638814,Zhou_2022}.
Prominent algorithms in MARL include IQL \citep{Tan2003MultiAR}, MADDPG \citep{lowe2020multiagent}, COMA \citep{foerster2017counterfactual}, MAPPO \citep{yu2022surprising}, VDN \citep{sunehag2017valuedecomposition}, and QMIX \citep{rashid2018qmix}. Theoretical research has made great process in reducing the sample complexity \citep{pmlr-v195-wang23b, xiong2023sampleefficientmultiagentrloptimization}.

\paragraph{Offline MARL.}
Offline MARL is a practical solution for handling sophisticated multi-agent environments. 
Empirically, to address issues related to out-of-distribution actions and complex reward functions, previous works have developed algorithms such as MABCQ \citep{jiang2023offline}, ICQ-MA \citep{yang2021believe}, OMAR \citep{pan2022plan}, and OMIGA \citep{wang2023offline}, which incorporate regularization or constraints on these actions and functions. 
MOMA-PPO \citep{barde2024modelbased} is a model-based approach to offline MARL that generates synthetic interaction data from offline datasets. 
\cite{NEURIPS2022_01d78b29} combines knowledge distillation with multi-agent decision transformers \citep{meng2022offline} for offline MARL.
Theoretical understanding of offline MARL, particularly in the context of Markov games, has been advanced by works that provide sample complexity guarantees for learning equilibria \cite{sidford2019solving,cui2020minimax,zhang2023modelbased,zhang2020finitesample,abe2020offpolicy,cui2022offline,cui2022provably,Zhang2023OfflineLI,Blanchet2023DoublePI,Shi2023ProvablyEO,zhong2022pessimistic}.

%% file: text/Preliminaries.tex
\section{Preliminaries}

\textbf{General-sum Markov Games.} We consider an episodic time-inhomogeneous general-sum Markov game $\M$, consisting of $m$ players, a shared state space $\S$, an individual action space $\A_i$ for each player $i\in[m]$ and a joint action space $\A=\A_1\times \A_2\times\cdots \A_m$. The game has a time horizon $H$, an initial state $s_1$, state transition probabilities $\P=(\P_1,\P_2,\cdots,\P_H)$ with $\P_h:\S\A\rightarrow \Delta(\S)$, and rewards $R={R_h(\cdot\mid s_h,a_h)}_{h=1}^H$ where $R_{h,i}\in[0,1]$ represents the random reward for player $i$ at step $h$. 
At each step $h\in[H]$, all players observe current state $s_h$ and simultaneously choose their actions $\a_h=(a_{h,1},a_{h,2},\cdots,a_{h,m})$. The next state $s_{h+1}$ is then sampled from $\P_h(\cdot\mid s_h,\a_h)$, and the reward $r_{h,i}$ for player $i$ is sampled from $R_{h,i}(\cdot\mid s_h,\a_h)$. The game terminates at step $H+1$, with each player aiming to maximize the total collected rewards.

We use $\pi=(\pi_1,\pi_2,\cdots,\pi_m)$ to denote a joint policy, where the individual policy for player $i$ is represented as $\pi_i=(\pi_{1,i},\pi_{2,i},\cdots,\pi_{H,i})$, with each $\pi_{h,i}:S\rightarrow\Delta(A_i)$ defined as the Markov policy for player $i$ at step $h$. 
The state value function and state-action value function for each player $i\in[m]$ are defined as
$$V_{h,i}^{\pi}(s_h):=\E_\pi\Mp{\sum_{t=h}^Hr_{t,i}(s_t,\a_t)\mid s_h},\ Q_{h,i}^{\pi}(s_h):=\E_\pi\Mp{\sum_{t=h}^Hr_{t,i}(s_t,\a_t)\mid s_h,\a_h},$$
where $\E_{\pi}=\E_{s_1,\a_1,\r_1,\cdots,s_{H+1}\sim \pi,\M}$ denotes the expectation over the random trajectory generated by policy $\pi$. The best response value for player $i$ is defined as
$$V_{h,i}^{\dagger, \pi_{-i}}(s_h):=\max_{\pi_i}V_{h,i}^{\pi_i,\pi_{-i}}(s_h),$$
which represents the maximal expected total return for player $i$ given that the other players follow policy $\pi_{-i}$.

A Nash equilibrium is a policy configuration where no player has an incentive to change their policy unilaterally. Formally, we measure how closely a policy approximates a Nash equilibrium using the \textit{Nash-Gap}:
$$\text{Nash-Gap}(\pi):=\sum_{i\in[m]}\Mp{V_{1,i}^{\dagger,\pi_{-i}}(s_1)-V_{1,i}^\pi(s_1)}.$$
By definition, the \text{Nash-Gap} is always non-negative, and it quantifies the potential benefit each player could gain by unilaterally deviating from the current policy. A policy $\pi$ is considered an $\epsilon$-Nash equilibrium \textit{iff} $\text{Nash-Gap}(\pi)\leq\epsilon$.

\textbf{Offline Multi-agent Reinforcement Learning with Preference Feedback.} In offline MARL with Preference Feedback, the algorithm has access to a pre-collected preference dataset generated by an unknown behavior policy interacting with an underlying Markov game. 
We consider two sampled trajectories, $\tau=(s_1,\a_1,s_2,\a_2,\cdots,s_{H+1})$ and $\tau'=(s'_1,\a'_1,s'_2,\a'_2,\cdots,s'_{H+1})$, drawn from distribution $\P(s_1,\a_1,s_2,\cdots,s_{H+1})=\Pi_h\pi^b(\a_h\mid s_h)\P(s_{h+1}\mid s_h,\a_h)$ induced by the behavior policy $\pi^b$. 
In MARLHF, the reward signal is not revealed in the dataset. Instead, each player can observe a binary signal $y_i$ from a Bernoulli distribution following the Bradley-Terry-Luce model \citep{bradley1952rank}:
$$\P(y_i=1\mid \tau,\tau')=\frac{\exp(\sum_{h=1}^Hr_i(s_h,\a_h))}{\exp(\sum_{h=1}^Hr_i(s_h,\a_h))+\exp(\sum_{h=1}^Hr_i(s'_h,\a'_h))},\forall i\in[m].\label{equ: Preference Signal}$$

We make the standard linear Markov game assumption \citep{zhong2022pessimistic}:
\begin{assumption}
$\M$ is a linear Markov game with a feature map $\psi:\S\times\A\rightarrow\R^d$ if we have
    $$\P_h(s_{h+1}\mid s_h,\a_h)=\inner{\psi(s_h,\a_h),\mu_h(s_{h+1})},\forall (s_h,\a_h,s_{h+1},h)\in\S\times\A\times\S\times[H],$$ $$r_i(s_h,\a_h)=\inner{\psi(s_h),\theta_{h,i}},\forall (s_h,\a_h,h,i)\in\S\times\A\times[H]\times[m],$$
    where $\mu_h$ and $\theta_{h,i}$ are unknown parameters. Without loss of generality, we assume $\Norm{\psi(s,\a)}\leq 1$ for all $(s,\a)\in\S\times\A$ and $\Norm{\mu_h(s)}\leq\sqrt{d},\Norm{\theta}_h\leq\sqrt{d}$ for all $h\in[H]$.
\end{assumption}

The one-hot feature map is defined as $\overline{\psi}_h(s,\a):=[0,\cdots,0,\psi(s,\a),0,\cdots,0]\in\R^{Hd}$, where $\psi(s,\a)$ is at position $(h-1)d+1$ to $hd$.


\textbf{Value-Decomposition Network (VDN).} 
In our experiments, we utilize VDN as an offline MARL algorithm for its effectiveness and simplicity. 
VDN~\citep{sunehag2017valuedecomposition} is a Q-learning style MARL architecture for cooperative games. It takes the idea of decomposing the team value function into agent-wise value functions, expressed as: 
$Q_h(s, \a) = \sum_{i=1}^n Q_{h,i}(s, a_i).$
In our experiments, we applied Deep Q-Network (DQN) \citep{mnih2013playing} with VDN to learn the team Q function. We chose DQN to maintain the simplicity and controllability of the experimental pipeline, which facilitates a more accurate investigation of the impact of various techniques on the learning process.

%% file: text/DatasetTheory.tex
\section{Dataset Coverage Theory for MARLHF}\label{section:theory}
In this section, we study the dataset coverage assumptions for offline MARLHF. 
For offline single-agent RLHF, \cite{zhu2023principled,zhan2023provable} show that single policy coverage is sufficient for learning the optimal policy. However, we prove that this assumption is insufficient in the multi-agent setting by constructing an counterexample. In addition, we prove that unilateral policy coverage is adequate for learning the Nash equilibrium.

\subsection{Policy Coverages}
We quantify the information contained in the dataset using covariance matrices, as the rewards and transition kernels are parameterized by a linear model. With a slight abuse of the notation, for trajectory $\tau=(s_1,\a_1,s_2,\a_2,\cdots,s_{H+1})$, we use $\psi(\tau):=[\psi(s_1,\a_1),\psi(s_2,\a_2),\cdots,\psi(s_H,\a_H)]$ to denote the concatenated trajectory feature. 
The reward coverage is measured by the preference covariance matrix:
$$\Sigma_{\D}^r=\lambda I+\sum_{(\tau,\tau')\in\D}(\psi(\tau)-\psi(\tau'))(\psi(\tau)-\psi(\tau'))^\top,$$
where $\psi(\tau)-\psi(\tau')$ is derived from the preference model. 
Similarly, the transition coverage is measured by the covariance matrix:
$$\Sigma_{\D,h}^{\P}=\lambda I+\sum_{(\tau,\tau')\in\D}\Mp{\psi(s_h,\a_h)\psi(s_h,\a_h)^\top+\psi(s'_h,\a'_h)\psi(s'_h,\a'_h)^\top}.$$
For a given state and action pair $(s_h,\a_h)$, the term $\Norm{\overline{\psi}_h(s_h,\a_h)}_{[\Sigma_{\D}^r]^{-1}}$ measures the uncertainty in reward estimation and $\Norm{\psi(s_h,\a_h)}_{[\Sigma_{\D,h}^{\P}]^{-1}}$ measures the uncertainty in transition estimation.
As a result, the overall uncertainty of a given policy $\pi$ with dataset $\D$ is measured by 
$$U_\D(\pi):=\E_{\pi}\Mp{\sum_{h=1}^H\Norm{\overline{\psi}_h(s_h,a_h)}_{[\Sigma_{\D}^r]^{-1}}+\sum_{h=1}^H\Norm{\psi(s_h,a_h)}_{[\Sigma_{\D,h}^{\P}]^{-1}}}.$$

\begin{definition}
For a Nash equilibrium $\pi^*$, different policy coverages are measured by the following quantities:
\begin{itemize}
    \item Single policy coverage: $U_\D(\pi^*)$.
    \item Unilateral policy coverage: $\max_{i,\pi_i}U_\D(\pi_i,\pi_{-i}^*)$.
    \item Uniform policy coverage: $\max_{\pi}U_\D(\pi)$.
\end{itemize}
Intuitively, small $U_\D(\pi^*)$ indicates that the dataset contains adequate information about $\pi^*$. A small $\max_{i,\pi_i}U_\D(\pi_i,\pi_{-i}^*)$ implies that the dataset covers all of the unilateral deviations of $\pi^*$, and small $\max_{\pi}U_\D(\pi^*)$ suggests that the dataset covers all possible policies.
\end{definition}

\subsection{Single Policy Coverage is Insufficient}
Our objective is to learn a Nash equilibrium policy from the dataset, which necessitates that the dataset sufficiently covers the Nash equilibrium. In the single-agent scenario, if the dataset covers the optimal policy, pessimism-based algorithms can be employed to recover the optimal policy. 
However, previous work \citep{cui2022offline,zhong2022pessimistic} has demonstrated that single policy coverage is insufficient for offline MARL. We extend this result to the context of offline MARL with preference feedback, as follows:
\begin{theorem}
\label{thm: Single Policy Coverage}
    (Informal) If the dataset only has coverage on the Nash equilibrium policy (i.e. small $U_\D(\pi^*)$), it is not sufficient for learning an approximate Nash equilibrium policy.
\end{theorem}
The proof is derived by a reduction from standard offline MARL to MARLHF. 
Suppose that MARLHF with single policy coverage suffices, we could construct an algorithm for standard offline MARL, which leads to a contradiction. The formal statement and the detailed proof are deferred to Appendix~\ref{apd:single_policy}.

\subsection{Unilateral Policy Coverage is Sufficient}
While single policy coverage is too weak to learn a Nash equilibrium, uniform policy coverage, though sufficient, is often too strong and impractical for many scenarios. Instead, we focus on unilateral policy coverage, which offers a middle ground between single policy coverage and uniform policy coverage.

\begin{theorem}
\label{thm: Unilateral Policy Coverage}
    (Informal) If the dataset has unilateral coverage on the Nash equilibrium policy, there exists an algorithm that can output an approximate Nash equilibrium policy.
\end{theorem}

The detailed proof is deferred to Appendix~\ref{apd:unilateral}. We leverage a variant of Strategy-wise Bonus and Surrogate Minimization (SBSM) algorithm in \citep{cui2022provably} with modified policy evaluation and policy optimization subroutines. 
Intuitively, the algorithm identifies a policy that minimizes a pessimistic estimate of the Nash gap. As a result, if the dataset has unilateral coverage, the output policywill have a small Nash gap and serves as a good approximation of the Nash equilibrium.

%% file: text/Method.tex
\section{Algorithmic Techniques for Practical Performance}
In Section \ref{section:theory}, we provided a theoretical characterization of the dataset requirements for PbMARL. 
However, the algorithm used in Theorem \ref{thm: Unilateral Policy Coverage} is not computationally efficient. In this section, we propose a practical algorithm for PbMARL and validate our theoretical findings through experiments.

\subsection{High-level Methodology}
Our PbMARL pipeline consists of two phases: 
In the first step, we train a reward prediction model $\phi$ and approximate the behavior policy $\pi_b$ using imitation learning; in the second step, we then apply an MARL algorithm to maximize a combination of the KL-divergence-based reward and standardized predicted reward $r_\phi$, ultimately deriving the final policy $\pi_{\mathbf{w}}$. 


\textbf{Step 1: Reward Training and Dataset Modeling.}
Given the preference signals of trajectories, we use neural networks to predict step-wise rewards $r_\phi(s_h, a_h)$ for each agent, minimizing the loss defined in \eqref{equ: Reward Model Target}. 
The objective is to map $(s,a_i)$-pairs to reward values such that the team returns align with the preference signals. 
At the same time, in order to utilize distribution-based penalty term $\log{\pi_{b}(s,a)}$ to cope with the extrapolation error in offline learning, an imitation learner is trained over the entire dataset to model the behavior policy $\pi_b$.


\textbf{Step 2: Offline MARL.}
Although in this work, VDN is chosen as the MARL oracle, it should be noted that other MARL architectures are also applicable. With the reward model $r_\phi$ and the approximated dataset distribution learned in Step 1, we are now able to construct a virtual step-wise reward for each agent. The agents are then trained to maximize the target defined in (\ref{equ: Final Reward Function}).

Given this framework, additional techniques are required to build a strong practical algorithm, which we provide more details below.

\subsection{Reward Regularization}
\label{subsec: Reward regularization}
Compared to step-wise reward signals, preference signals are $H$ times sparser, making them more challenging for a standard RL algorithm to utilize effectively.
Concretely, this reward sparsity causes the naive optimization of the negative log-likelihood (NLL) loss to suffer from two key problems:
\begin{enumerate}
    \item \textbf{Sparse and spiky reward output.} When calculating NLL losses, spreading the reward signal along the trajectories is equivalent to summing it at the last time step (Figure ~\ref{fig:rewardlines_mixed_noMSE}). 
    However, a sparse reward signal is harder for traditional RL methods to handdle due to the lack of continuous supervision. More uniformly distributed rewards across the entire trajectory generally leads to more efficient learning in standard RL algorithms.
    \item \textbf{Over-reliance on irrelevant features.} The model may exploit redundant features as shortcuts to predict rewards. For instance, expert agents in cooperative games usually exhibit a fixed pattern of collaboration from the very beginning of the trajectory (such as specific actions or communication moves). The reward model might use these patterns to differentiate them from agents of other skill levels, thereby failing to capture the true reward-observation causal relationships. 
\end{enumerate}

To mitigate these problems, we introduce an extra Mean Squared Error (MSE) regularization along the time axis (Equation \ref{equ: Reward Model Target}, \ref{equ: Reward Regularization}). By limiting the sudden changes in reward predictions between adjacent time steps, this regularization discourages the reward model from concentrating its predictions on just a few time steps. 
While these issues can also be mitigated by using more diversified datasets and adding regularization to experts to eliminate reward-irrelevant action patterns, these approaches can be costly and sometimes impractical in real-world applications.  
In contrast, our MSE regularization is both easy to implement and has been empirically verified to be effective, creating more uniform reward distribution (Figure \ref{fig:rewardlines}) and better performances. 

\begin{equation}
\label{equ: Reward Model Target}
    L_{\text{RM}}(\phi) = -\mathbb{E}_{\mathcal{D}}\left[\sum_{i=1}^m\log\sigma(y_i(r_{\phi,i}(\tau_1)- r_{\phi,i}(\tau_2)))\right]  + \frac{\alpha}{\text{Var}_{\mathcal{D}}(r_ \phi)} L_{\text{MSE}}(\phi, \tau),
\end{equation}
where the regularization term $L_{\text{MSE}}$ is defined as:
\begin{equation}
\label{equ: Reward Regularization}
    L_{\text{MSE}}(\phi,\tau) = \mathbb{E}_{\mathcal{D}}\left[\sum_{h=1}^{H-1}\|r_\phi(s_h,\a_h) - r_{\phi}(s_{h+1}, \a_{h+1})\|_2^2\right].
\end{equation}
Here $\alpha$ is the regularization coefficient, which is set to be 1 in our experiments. The variance of $r_\phi$ is calculated over the training set to adaptively scale the regularization term. 
During training, $\text{Var}_\mathcal{D}(r_\phi)$ is detached to prevent gradients from flowing through it. The effectiveness of this method is validated in the ablation study (cf. Section \ref{subsec: ablation studies-reward regularization}).
\vspace{-2mm}

\subsection{Dataset Distribution-Based Pessimism}
\label{subsec: Imitation Learning}
There are various methods to mitigate the over-extrapolation errors in offline RL \citep{peng2019advantageweighted, nair2021awac}, including conservative loss over the Q-function \citep{kumar2020conservative} and directly restricting the learned policy actions to those within within the dataset \citep{fujimoto2019offpolicy}. 
We add a per-step dataset-based penalty term, $\log{\pi_b (s, \a)}$, as pessimism towards less explored states.
Imitation learning is utilized to estimate the behavior policy $\pi_{b}$ from the dataset distribution.
To stabilize training, we standardize predicted reward $r_\phi$ over $\mathcal{D}$ before combining it with the penalty term to make them comparable:

\vspace{-2mm}
\begin{equation}
\label{equ: Final Reward Function}
    \text{objective}(\mathbf{w}) = \mathbb{E}_{\tau\sim\pi_\mathbf{w}}\left[\sum_{h=1}^H r_\text{std}(s_h,\a_h, \phi) + \beta \log{\pi_b (s_h, \a_h)}\right],
\end{equation}
where $\beta$ is the pessimism coefficient.
\footnote{$\beta$ is set to be $(1,1,10,10)$ in Spread-v3, Reference-3, Tag-v3 and Overcooked respectively in the main experiments, and the pessimism term $\beta \log{\pi_b (s_h, \a_h)}$ is clipped to (-10, 1) in practice.} 
The standardized reward $r_\text{std}$ is defined as:

\begin{equation}
    r_\text{std}(s_h, \a_h, \phi) = \sum_{i=1}^m\frac{r_{\phi}(s_h,a_{h,i}) - \mathbb{E}_{\mathcal{D}}({r}_{\phi})}{\sqrt{\text{Var}_{\mathcal{D}}(r_\phi)}}.
\end{equation}


Intuitively, the penalty term $ \log{\pi_b (s_h, \a_h)}$ discourages the agents from deviating from the most preferred actions in the dataset. 
The effectiveness of this method is validated in the ablation study (cf. Section \ref{subsec: other discussion}).


%% file: text/Experimental.tex
\section{Experiments}
\label{sec: Experiments}
\newcommand{\std}[1]{\tiny{$\pm$ #1}}

We design a series of experiments to validate our theories and methods in common general-sum games.
Specifically, we first use online RL algorithms to train expert agents, and take intermediate checkpoints as rookie agents. Then, we use these agents to collect datasets and use the Bradley-Terry model over standardized returns to simulate human preference.
Experiments are carried out to verify the efficiency of our approach with unilateral policy dataset coverage (in Theorem \ref{thm: Unilateral Policy Coverage}) while single policy coverage is insufficient (stated in Theorem \ref{thm: Single Policy Coverage}).
We also design ablation studies to showcase the importance of our methods, particularly focusing on reward regularization and dataset distribution-based pessimism. 


\subsection{Environments}
Our experiments involved three Multi-Agent Particle Environments (MPE), including Spread-v3, Tag-v3 and Reference-v3, and Overcooked environment implemented with JaxMARL codebase \citep{flair2023jaxmarl}. 
\textbf{Spread-v3} contains a group of agents and target landmarks, where the objective is to cover as many landmarks as possible while avoiding collisions. 
\textbf{Tag-v3} contains two opposing groups, where quicker "preys" need to escape from  "predators". To ensure a fair comparison of different predator cooperation policies, we fixed a pretrained prey agent.
\textbf{Reference-v3} involves two agents and three potential landmarks, where the agents need to find each one's target landmark to receive a high reward. The target landmark of each agent is only known by the other agent at first.
\textbf{Overcooked} involves two agents moving and operating objects in a gridworld. 
A more detailed description of the tasks and their associated challenges is provided in Appendix \ref{apd: Tasks Descriptions}.

\subsection{The Importance of Dataset Diversity}
\label{subsec: datasets}
To study the influence of diversity of dataset, we manually designed 4 kinds of mixed joint behavior polices, and change their ratios to form different datasets. 
\begin{itemize}[leftmargin=2em,itemsep=0em,topsep=0em]

    \item Expert policy: $n$ expert agents. Trained with online RL algorithms till convergence.
    \item Rookie policy: $n$ rookie agents. Trained with online RL algorithms with early stop.
    \item Trivial policy: $n$ random agents. All actions are uniformly sampled from the action space.
    \item Unilateral policy: $n-1$ expert agents and $1$ rookie agent of different proficiency level. 
\end{itemize}

Table~\ref{tab: dataset mix ratio} presents the ratio of trajectories collected by the four different policies. The experiments are designed to hierarchically examine the roles of diversity (Diversified vs. Mix-Unilateral), unilateral coverage (Mix-Unilateral vs. Mix-Expert), and trivial comparison (Mix-Expert vs. Pure-Expert).

The ranking of diversity follows the order:
\[
\text{Pure-Expert} < \text{Mix-Expert} < \text{Mix-Unilateral} < \text{Diversified}
\]

Due to the inherent limitations of offline reinforcement learning (RL) in action selection dictated by the dataset, the effectiveness of learning is often strongly correlated with dataset quality, i.e. the level of expertise demonstrated in the dataset. However, the results in preference-based MARL experiments partially diverge from this conventional conclusion. While the quality of the dataset remains critical, experiments on Reference-v3 and Overcooked (Table \ref{tab:exp-mpe}) indicate that diversity and unilateral data can significantly enhance the performance of the reward model, thereby facilitating learning.

The main experimental results are presented in Table~\ref{tab:exp-mpe} and Table~\ref{tab:exp-bcqiql}.Among all the experiments, apart from the experiments on Tag-v3, where the high operational precision requirements make data quality more critical than diversity, the other three environments validate our conclusions across all algorithms.

\begin{table*}[t]
\centering
\begin{tabular}{lrrrr}
\toprule
 & Expert & Unilateral & Rookie & Trivial  \\
\midrule
Diversified & 1 & 1 & 1 &1\\
Mix-Unilateral & 2 &1 & 0 & 1 \\
Mix-Expert & 3 &0 & 0 & 1 \\
Pure-Expert & 4 & 0 & 0 & 0 \\
\bottomrule
\end{tabular}
\caption{{Final datasets mixed with various ratios. The overall dataset size is kept to 38400 trajectories for MPE, and 960 trajectories for Overcooked. (cf. Appendix \ref{apd: dataset}})
}
\label{tab: dataset mix ratio}
\vspace{-3mm}
\end{table*}

\begin{table*}[t]
\centering

\begin{tabular}{l|lllll}
\toprule
Algorithm & Dataset &  Spread-v3 & Tag-v3 & Reference-v3 & Overcooked \\
\midrule
VDN with  & Diversified & -21.16 \std{0.54}  &29.28 \std{1.08}   &-18.89 \std{0.60}   & \textbf{238.89 \std{3.50}}\\
Pessimism Penalty &Mix-Unilateral &-21.03 \std{0.44} &36.65 \std{0.70} & -18.80 \std{0.63} & 221.80 \std{26.66} \\
&Mix-Expert & -20.98 \std{0.54} &35.96 \std{0.86} & -18.80 \std{0.44} & 35.26 \std{55.19} \\
&Pure-Expert & -21.01 \std{0.57}&\textbf{39.55 \std{0.77}} & -28.97 \std{2.89} & 3.36 \std{7.19} \\
\bottomrule
\end{tabular}
\caption{In the simplest environment, Spread-v3, different dataset gives similar performance. In Tag-v3 environment, where precise actions are required, the quality of the dataset (proportion of expert demonstration) is more important than diversity. In contrast, in Overcooked environment, which focuses on strategy learning and demands less on precision, dataset diversity contributes to improved stability, with Unilateral playing a particularly critical role. In the Reference-v3 environment, which balances the need for precision and strategic, the importance of both factors is more balanced, but non-expert data is still necessary.}
\label{tab:exp-mpe}
\vspace{-5mm}
\end{table*}

\begin{figure}[H]
    \centering
    \begin{subfigure}[b]{0.4\textwidth}
        \raisebox{0.28cm}{
        \includegraphics[width=\textwidth]{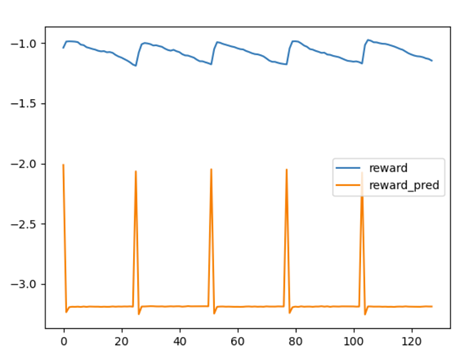}
        }
        \caption{Mix-Expert $\alpha=0$ (spread-v3)}
        \label{fig:rewardlines_mixed_noMSE}
    \end{subfigure}
    \hfill
    \begin{subfigure}[b]{0.55\textwidth}
        \includegraphics[width=\textwidth]{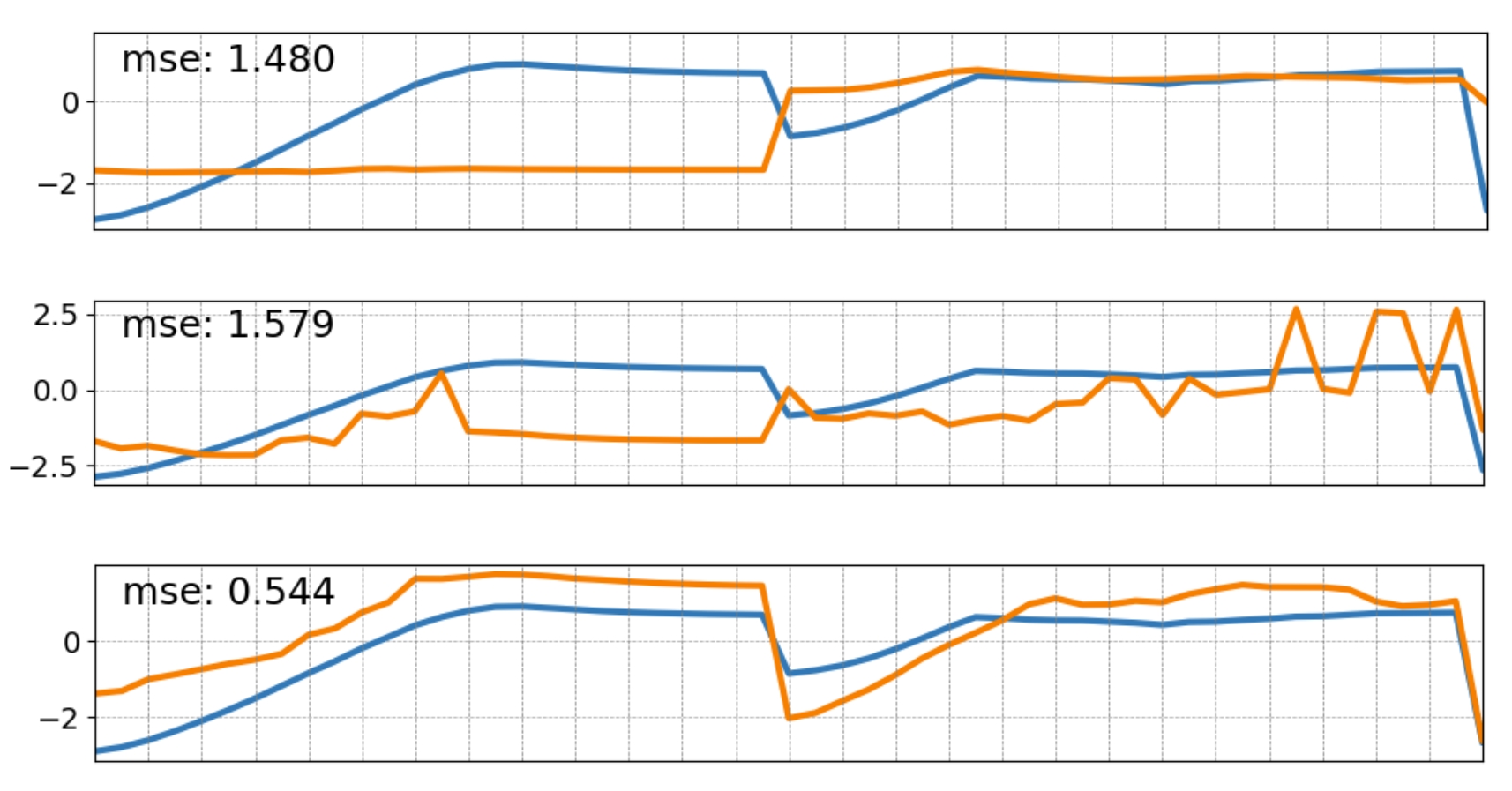}
        \subcaption{reference-v3}
        \subcaption*{(b1) Pure-Expert $\ $ (b2) Diversified $\alpha=0$ $\ $ (b3) Diversified}
        \label{fig:rewardlines_3}
    \end{subfigure}
    \caption{(a) Averaged reward predictions and ground truth of a trajectory sample on spread-v3. (b) Standardized reward predictions and ground truth of a trajectory sample in reference-v3. When trained with expert data only (b1), $\phi$ experiences a mode collapse, failing to give informative signals. Reward function trained without regularization (b2) shows spiky patterns and tends to accumulate predictions at certain time steps when trained with less diversified datasets as (a). Our method with diversified dataset (b3) gives predictions that approximate the ground truth well.}
    \label{fig:rewardlines}
    \vspace{-5mm}
\end{figure}

\begin{table*}[t]
\centering
\begin{tabular}{l|lllll}
\toprule
Algorithm &Dataset &  Spread-v3 & Reference-v3 & Overcooked \\
\midrule
MAIQL & Diversified &-25.33 \std{1.40}   &\textbf{-22.15 \std{0.55}}   &\textbf{16.59  \std{11.22}}   \\
    &Mix-Unilateral &-23.25 \std{1.06}   &-23.22 \std{1.37}   &0.00   \std{0.00}   \\
    &Mix-Expert     &-23.26 \std{0.90}   &-24.21 \std{1.60}   &0.00   \std{0.00}   \\
    &Pure-Expert    &-26.01 \std{1.53}   &-29.47 \std{1.65}   &0.00   \std{0.00}\\
    \midrule 
MABCQ & Diversified &-20.02 \std{0.64}   &\textbf{-17.64 \std{0.43}}   &\textbf{239.34 \std{1.67}}\\
    &Mix-Unilateral &-19.47 \std{0.33}   &-17.64 \std{1.11}   &215.01 \std{65.43}\\
    &Mix-Expert     &-19.42 \std{0.17}   &-17.88 \std{0.78}   &50.32  \std{82.82}\\
    &Pure-Expert    &-20.56 \std{0.38}   &-25.90 \std{1.11}   &1.14   \std{3.46}\\
\bottomrule
\end{tabular}
\caption{Test returns of MAIQL and MABCQ. In the experimental results, we can observe a clear preference toward more diversified datasets. 
Compared to our method and BCQ, which directly calculate $\max_a Q$ for Bellman updates, IQL employs expectile regression to estimate it. So MAIQL demands higher accuracy of the reward model. Consequently, the performance improvements brought by dataset diversity are also more pronounced in MAIQL experiments. }
\label{tab:exp-bcqiql}
\end{table*}

\begin{table*}[t]
\centering
\begin{tabular}{lrrrrrr}
\toprule
& $\beta=0$ & $\beta=0.1$ & $\beta=1$ & $\beta=10$ & $\beta=100$ & $\alpha=0$ \\
\midrule
Spread-v3 & -22.56 \std{1.61} & -22.03 \std{0.67} & -20.82 \std{0.53} & -20.46 \std{0.51} & -20.35 \std{ 0.43} & -22.21 \std{0.72} \\
Tag-v3 & 4.11 \std{1.66} & 4.25 \std{0.53} & 10.96 \std{1.20} & 28.88 \std{1.02} & 29.53 \std{1.35} & 30.77 \std{0.57} \\
Reference-v3 & -19.69 \std{0.36} & -19.37 \std{0.53} & -18.89 \std{0.78} & -18.33 \std{0.42} & -18.54 \std{0.46} & -21.86 \std{0.73} \\
Overcooked& 0.00 \std{0.00} &0.00 \std{0.00} & 149.53 \std{86.74} & 238.89 \std{3.50} & \textbf{240 \std{0.00}} & \textbf{240 \std{0.00}}\\
\bottomrule
\end{tabular}
\caption{Comparison of test return with different hyperparameters. Standard pipeline take pessimism coefficient $\beta=1$ for Spread-v3, Reference-v3 and $\beta=10$ for Tag-v3, Overcooked, and the MSE reward regularization coefficient $\alpha$ is set to the optimal value for fixed $\beta$. All the agents are trained on Diversified Dataset across 10 random seeds. Results show that larger $\beta$ always gives better performance and a proper positive $\alpha$ can improve performance. 
}
\label{tab: tech-mpe}
\vspace{-3mm}
\end{table*}

\subsection{Other Ablation Studies}
\label{subsec: other discussion}

\paragraph{Reward regularization}
\label{subsec: ablation studies-reward regularization}
In Figure~\ref{fig:rewardlines}, we examined the effectiveness of our proposed reward regularization technique.
Figure~\ref{fig:rewardlines_mixed_noMSE} demonstrates that without regularization, the learned rewards tend to be sparse and spiky compared to the ground truth rewards.
We also observe that the rewards often exhibit temporal continuity, which can create greater discrepancies with the sparse, pulse-like ground truth. 
Notably, we found that adding stronger regularization does not necessarily lead to underfitting of the reward model; in some cases, it even helps the model converge to a lower training loss. Detailed parameters and experimental results are provided in the appendix (cf. Table \ref{tab: mse-ablation}). We attribute this to the role of regularization in preventing the model from overly relying on shortcuts.

\paragraph{Pessimism coefficient} Due to the clipping in Equation \ref{equ: Final Reward Function}, excessively large \(\beta\) values will not dominate the entire reward function. As a result, larger \(\beta\) values almost never degrade the agent's performance in our experiments (Table~\ref{tab: tech-mpe}). This allows us to increase \(\beta\) with relative confidence. Therefore, we generally recommend setting \(\beta\) to a value between 10 and 100 for optimal performances.

\paragraph{Scalability}
We also tested the scalability on Spread-v3. While our current approach manages the scaling of agents without introducing new problems, it does not specifically address the inherent issues of instability and complexity that are well-documented in traditional MARL (cf. Appendix~\ref{apd: scalability analysis}).

%% file: text/Discussions.tex
\section{Discussion}

In this paper, we proposed dedicated algorithmic techniques for offline PbMARL and provided theoretical justification for the unilateral dataset coverage condition. We believe our work is a significant step towards systematically studying PbMARL and offers a foundational framework for future research in this area. The flexibility of our framework allows for application across a wide range of general games, and our empirical results validate the effectiveness of our proposed methods in various scenarios.

Looking ahead, there is significant potential to extend this work to more complex, real-world scenarios, particularly by integrating Large Language Models (LLMs) into multi-agent systems. Future research will focus on fine-tuning and aligning LLMs within PbMARL, addressing challenges such as increased complexity and the design of effective reward structures.

%% file: text/Appendix_Proof.tex
\section{Missing Proofs in Section \ref{section:theory}}\label{apd:proof}

\subsection{Single Policy Coverage is Insufficient}
\label{apd:single_policy}

\begin{table}[h]
    \centering
    \begin{minipage}{0.45\textwidth}
        \centering
        \begin{tabular}{|c|cc|}
        \hline
          & $a_1$ & $a_2$ \\
        \hhline{---}
        $b_1$ & 0.5 & 1 \\
        $b_2$ & 0 & 0.5 \\
        \hline
        \end{tabular}
    \end{minipage}
    \hfill
    \begin{minipage}{0.45\textwidth}
        \centering
        \begin{tabular}{|c|cc|}
        \hline
          & $a_1$ & $a_2$ \\
        \hhline{---}
        $b_1$ & 0.5 & 0 \\
        $b_2$ & 1 & 0.5 \\
        \hline
        \end{tabular}
    \end{minipage}
    \vspace{0.3cm}
    \caption{Here we present two matrix games, $\mathcal{M}_1$ (left) and $\mathcal{M}_2$ (right). The row player aims to maximize their reward, while the column player aims to minimize it. The Nash Equilibrium in $\mathcal{M}_1$ is $(a_1, b_1)$, and in $\mathcal{M}_2$ it is $(a_2, b_2)$. Note that with a dataset covering only these two states, it is impossible to distinguish between these two games, and therefore, it is not possible to identify the exact Nash Equilibrium.}
    \label{apd: theorem 1 example}
\end{table}

\begin{theorem}
    (Restatement of Theorem \ref{thm: Single Policy Coverage}) For any algorithm and constant $C>0$, there exists a Markov game and a compliant dataset with $U_{\D}(\pi^*)\leq C$ such that the output policy is at most an $0.5$-Nash equilibrium.
    \label{thm: restatement of single policy}
\end{theorem}

\begin{proof}
    We construct two linear Markov games with a shared compliant dataset such that no policy is a good approximate Nash equilibrium in both Markov games. Similar to \citep{cui2022offline}, we consider Markov games with $H=1$, $m=2$, $\A_1=\{a_1,a_2\}$ and $\A_2=\{b_1,b_2\}$ with deterministic reward presented in Table \ref{apd: theorem 1 example}. 
    
    The feature mapping for these two games is 
    $$\psi(a_1,b_1)=e_1, \psi(a_1,b_2)=e_2, \psi(a_2,b_1)=e_3, \psi(a_2,b_2)=e_4,$$
    where $e_i\in\R^4$ are the unit base vectors. Directly we have the reward parameters $\theta$ as the rewards.
    
    The behavior policy is $\pi^b(a_1,b_1)=\pi^b(a_2,b_2)=1/2$ and dataset is $\D=\{(\tau_i,\tau'_i,y_i)\}_{i=1}^n$ with
    $$\tau_i,\tau'_i\in\{(a_1,b_1),(a_2,b_2),y_i\sim \mathrm{Ber}(\exp(r_1(\tau_i)-r_1(\tau'_i)))\}.$$

    As the dataset covers the Nash equilibrium for both games, with enough samples, we have $U_{\D}(\pi^*)\leq C$ for any constant $C$. Suppose the output policy of the algorithm is $\pi=(\mu,\nu)$, then $\pi$ is at most 0.5-Nash equilibrium in one of these two games \footnote{See proof for Theorem 3.3 in \citep{cui2022offline}}.
\end{proof}


\subsection{Unilateral Policy Coverage}
\label{apd:unilateral}
\begin{algorithm}[ht!]
    \caption{Value Estimation}
    \label{algo:value estimation}
    \begin{algorithmic}[1]
        \State {\bfseries Input}: Offline dataset $\D$, player index $i$, policy $\pi$.
        \State {\bfseries Initialization}: $\underline{V}_{H+1,i}^\pi(s)=0$.
        \For{$h=H,H-1,\cdots,1$}
        \State $w_{h,i}=[\Sigma_{\D,h}^\P]^{-1}\sum_{n=1}^N\psi(s_h^n,\a_h^n)[\underline{r}_{h,i}(s_h^n,\a_h^n)+\underline{V}^{\pi}_{h+1,i}(s_{h+1}^n)]$.
        \State $\underline{Q}^{\pi}_{h,i}(\cdot,\cdot)=\max\{\inner{\psi(\cdot,\cdot),w_{h,i}}-C_{\P}[\psi(\cdot,\cdot)^\top[\Sigma_{\D,h}^\P]^{-1}\psi(\cdot,\cdot)]^{1/2},0\}$
        \State $\underline{V}^{\pi}_{h,i}(\cdot)=\E_{a\sim\pi_h(\cdot)}\underline{Q}^{\pi}_{h,i}(\cdot,\a).$
        \EndFor
    \end{algorithmic}
\end{algorithm}

\begin{algorithm}[ht!]
    \caption{Best Response Estimation}
    \label{algo:best response estimation}
    \begin{algorithmic}[1]
        \State {\bfseries Input}: Offline dataset $\D$, player index $i$, policy $\pi_{-i}$.
        \State {\bfseries Initialization}: $\overline{V}_{H+1,i}^{\dagger,\pi_{-i}}(s)=0$.
        \For{$h=H,H-1,\cdots,1$}
        \State $w_{h,i}=[\Sigma_{h,\D}^\P]^{-1}\sum_{n=1}^N\psi(s_h^n,\a_h^n)[\overline{r}_h(s_h^n,\a_h^n)+\overline{V}^{\dagger,\pi_{-i}}_{h+1,i}(s_{h+1}^n)]$.
        \State $\overline{Q}^{\dagger,\pi_{-i}}_{h,i}(\cdot,\cdot)=\min\{\inner{\psi(\cdot,\cdot),w_{h,i}}+\beta[\psi(\cdot,\cdot)^\top[\Sigma_{h,\D}^\P]^{-1}\psi(\cdot,\cdot)]^{1/2},H\}$
        \State $\overline{V}^{\dagger,\pi_{-i}}_{h,i}(\cdot)=\max_{a_i\in\A_i}\E_{\a_{-i}\sim\pi_{h,-i}(\cdot)}\overline{Q}^{\dagger,\pi_{-i}}_{h,i}(\cdot,\a).$
        \EndFor
    \end{algorithmic}
\end{algorithm}

\begin{algorithm}[ht!]
    \caption{Surrogate Minimization}
    \label{algo:surrogate minimization}
    \begin{algorithmic}[1]
        \State {\bfseries Input}: Offline dataset $\D$.
        \State {\bfseries Initialization}: Algorithm \ref{algo:value estimation} for computing $\underline{V}_{1,i}^{\pi}$ and Algorithm \ref{algo:best response estimation} for computing $\overline{V}_{1,i}^{\dagger,\pi_{-i}}(s_1)$.
        \State {\bfseries Output}: $\pi^{\mathrm{output}}=\argmin_{\pi}\sum_{i\in[m]}\Mp{\overline{V}_{1,i}^{\dagger,\pi_{-i}}(s_1)-\underline{V}_{1,i}^{\pi}(s)}$
    \end{algorithmic}
\end{algorithm}

In our framework, we utilize Maximum Likelihood Estimation (MLE) to estimate the reward function for each player. For simplicity, we omit the subscript $i$ for player $i$. 
Note that the reward function can be expressed as $r_{\theta}(\tau)=\sum_{h=1}^Hr_{h}(s_h,\a_h)=\inner{\psi(\tau),\theta}$, where $\theta=[\theta_1,\theta_2,\cdots,\theta_H]$ represents the parameters we aim to optimize. 
At each step $h$, we minimize the NLL loss:
$$\widehat{\theta}=\argmin_{\theta_h\leq\sqrt{d},h\in[H]}-\sum_{n=1}^N\log\Sp{\frac{1(y^n=1)\exp(r_\theta(\tau))}{\exp(r_\theta(\tau))+\exp(r_\theta(\tau'))}+\frac{1(y^n=0)\exp(r_\theta(\tau'))}{\exp(r_\theta(\tau))+\exp(r_\theta(\tau'))}}.$$
This optimization problem helps in learning a reward function that aligns well with the observed data.

By Lemma \ref{lemma:reward concentration}, we establish a confidence region for the estimated parameters $\theta$: 
$$\Theta=\Bp{\theta:\Norm{\theta-\widehat{\theta}}_{\Sigma_{\D}^r+\lambda I}\leq C_{r}=C\sqrt{\frac{dH+\log(1/\delta)}{\lambda^2n}+d}}.$$
We define the optimistic reward and the pessimistic reward as follows:
$$\overline{r}_{h}(s,\a):=\max_{\theta\in\Theta}\inner{\psi(s,\a),\theta_h},\underline{r}_{h}(s,\a):=\min_{\theta\in\Theta}\inner{\psi(s,\a),\theta_h}.$$

We define the Bellman operator:
$$[\B_{h,i}V_{h+1,i}](s,\a)=r_{h,i}(s,\a)+\sum_{s'\in\S}P(s'\mid s,\a)V_{h+1,i}(s), \forall i\in[m], h\in[H], s\in\S, \a\in\A.$$

\begin{lemma}\label{lemma:reward bound}
    With probability at least $1-\delta$, we have
    $$r_{h,i}(s,\a)-2C_r\Norm{\overline{\psi}_h(s,\a)}_{[\Sigma_{\D}^r+\lambda I]^{-1}}\leq\underline{r}_{h,i}(s,\a)\leq r_{h,i}(s,\a)\leq \overline{r}_{h,i}(s,\a)\leq r_{h,i}(s,\a)+2C_r\Norm{\overline{\psi}_h(s,\a)}_{[\Sigma_{\D}^r+\lambda I]^{-1}}.$$
\end{lemma}
This lemma establishes bounds on the reward function $r_{h,i}(s,\a)$. Specifically, the optimistic reward and pessimistic reward are under constrains at a high probability, up to a margin of error determined by $C_r$ and the norms of the feature representation $\overline{\psi}_h(s,\a)$.

\begin{proof}
    We begin by referencing Lemma \ref{lemma:reward concentration}, which establishes that with probability at least $1 - \delta$, the estimated parameters $\theta$ reside within the confidence region $\Theta$. This allows us to assert the following relationships: 
    $$\underline{r}_{h,i}(s,\a)=\min_{\theta\in\Theta}\inner{\psi(s,\a),\theta_h}\leq r_{h,i}(s,\a)\leq\overline{r}_{h,i}(s,\a)=\max_{\theta\in\Theta}\inner{\psi(s,\a),\theta_h}.$$
    Next, we quantify the deviation between the optimistic reward and the true reward:
    \begin{align*}
        &\overline{r}_{h,i}(s,\a)-r_{h,i}(s,\a)\\
        =& \inner{\psi(s,\a),\overline{\theta}_h}-\inner{\psi(s,\a),\theta_h}\\
        =&\inner{\overline{\psi}_h(s,\a),\overline{\theta}-\widehat{\theta}}+\inner{\overline{\psi}_h(s,\a),\widehat{\theta}-\theta}\\
        \leq& \Norm{\overline{\psi}_h(s,\a)}_{[\Sigma_{\D}^r+\lambda I]^{-1}}\Norm{\overline{\theta}-\widehat{\theta}}_{\Sigma_{\D}^r+\lambda I}+\Norm{\overline{\psi}_h(s,\a)}_{[\Sigma_{\D}^r+\lambda I]^{-1}}\Norm{\widehat{\theta}-\theta}_{\Sigma_{\D}^r+\lambda I}\\
        \leq& 2C_{r}\Norm{\overline{\psi}_h(s,\a)}_{[\Sigma_{\D}^r+\lambda I]^{-1}}.
    \end{align*}
    A similar argument can be applied to show that the pessimistic reward is also bounded:
    $$r_{h,i}(s,\a)-\underline{r}_{h,i}(s,\a)\leq2C_{r}\Norm{\overline{\psi}_h(s,\a)}_{[\Sigma_{\D}^r+\lambda I]^{-1}}.$$
\end{proof}

\begin{lemma}\label{lemma:value opt/pes}
    With probability at least \( 1 - \delta \), the following bounds hold for all time steps \( h \in [H] \), for each agent \( i \in [m] \), for all states \( s \in \mathcal{S} \), and actions \( a \in \mathcal{A} \):
    \begin{align*}
        \underline{V}^{\pi}_{h,i}(s,a) &\leq V^{\pi}_{h,i}(s,a), \\
        \overline{V}^{\dagger,\pi_{-i}}_{h}(s,a) &\geq V^{\dagger,\pi_{-i}}_h(s,a).
    \end{align*}
\end{lemma}

\begin{proof}
    We prove the statements by mathematical induction. The base case for step \( H + 1 \) holds trivially, as all quantities are zero at this final step. 
    Suppose step $h+1$ holds and we consider step $h$. We will show they also hold for step \( h \). 
    For the first argument of pessimistic value, we have
    \begin{align*}
        \underline{V}_{h,i}^\pi(s)
        =&\E_{\a\sim\pi_h(s)}\underline{Q}^{\pi}_{h,i}(s,\a)\\
        \leq&\E_{\a\sim\pi_h(s)}\Mp{\B_{h,i}\underline{V}^{\pi}_{h+1,i}(s,\a)}\tag{Lemma \ref{lemma:opt/pes}}\\
        \leq&\E_{\a\sim\pi_h(s)}\Mp{\B_{h,i}V^{\pi}_{h+1,i}(s,\a)}\tag{Lemma \ref{lemma:value opt/pes}}\\
        =& V_{h,i}^{\pi}(s).
    \end{align*}
    For the second argument of optimistic value, we have
    \begin{align*}
        \overline{V}^{\dagger,\pi_{-i}}_{h,i}(s,\a)
        =&\max_{a_i\in\A_i}\E_{\a_{-i}\sim\pi_{h,-i}(s)}\overline{Q}^{\dagger,\pi_{-i}}_{h,i}(\cdot,\a)\\
        \geq&\max_{a_i\in\A_i}\E_{\a_{-i}\sim\pi_{h,-i}(s)}\Mp{\B_{h,i}\overline{V}^{\dagger,\pi_{-i}}_{h+1,i}(s,\a)}\tag{Lemma \ref{lemma:opt/pes}}\\
        \geq&\max_{a_i\in\A_i}\E_{\a_{-i}\sim\pi_{h,-i}(s)}\Mp{\B_{h,i}V^{\dagger,\pi_{-i}}_{h+1,i}(s,\a)}\tag{Lemma \ref{lemma:value opt/pes}}\\
        =&V^{\dagger,\pi_{-i}}_{h,i}(s,\a).
    \end{align*}
    These complete the induction step and shows the lemma holds for all steps.
\end{proof}

\begin{lemma}\label{lemma:value bound}
    With probability at least $1-\delta$, we have
    $$V_{1,i}^\pi(s_1)-\underline{V}_{1,i}^\pi(s_1)\leq\E_{\pi}\Mp{2C_{\P}\sum_{h=1}^H\Norm{\psi(s_h,\a_h)}_{[\Sigma_{\D,h}^{\P}+\lambda I]^{-1}}+2C_{r}\sum_{h=1}^H\Norm{\overline{\psi}(s_h,\a_h)}_{[\Sigma_{\D}^{r}+\lambda I]^{-1}}},$$
    $$\overline{V}_{1,i}^{\dagger,\pi_{-i}}(s_1)-V_{1,i}^{\dagger,\pi_{-i}}(s_1)\leq\E_{\pi_i^{\dagger},\pi_{-i}}\Mp{\sum_{h=1}^H2C_{\P}\Norm{\psi(s_h,\a_h)}_{[\Sigma_{\D}^{\P}+\lambda I]^{-1}}+2C_{r}\sum_{h=1}^H\Norm{\overline{\psi}(s_h,\a_h)}_{[\Sigma_{\D}^{r}+\lambda I]^{-1}}}.$$
\end{lemma}

  This lemma establishes bounds on the difference between the expected value of the policy and its lower or upper estimates. Using properties of the Bellman operator, we can show that the difference between \( V_{1,i}^\pi(s_1) \) and \( \underline{V}_{1,i}^\pi(s_1) \) is controlled by the expected norms of the feature representations \( \psi \) and \( \overline{\psi} \), scaled by constants \( C_{\mathcal{P}} \) and \( C_r \). 

\begin{proof}
We prove each bound separately, utilizing the results from Lemma \ref{lemma:opt/pes}.

For the first bound of pessimistic value estimation, we analyze the expected action-value \( Q^\pi_{1,i}(s_1,\a) \) and its lower estimate, applying the Bellman operator iteratively. We have
    \begin{align*}
        &V_{1,i}^\pi(s_1)-\underline{V}_{1,i}^\pi(s_1)\\
        =&\E_{\a\sim\pi_1(s_1)}Q^\pi_{1,i}(s_1,\a) -\E_{\a\sim\pi_1(s_1)}\underline{Q}^\pi_{1,i}(s_1,\a)\\
        \leq&\E_{\a\sim\pi_1(s_1)}\Mp{\B_{1,i}V^{\pi}_{2}(s_1,\a)-\B_{1,i}\underline{V}_{2}(s_1,\a)+2C_{\P}\Norm{\psi(s_1,\a)}_{[\Sigma_{\D,1}^{\P}+\lambda I]^{-1}}+2C_{r}\Norm{\overline{\psi}_1(s_1,\a)}_{[\Sigma_{\D}^{r}+\lambda I]^{-1}}}\\
        =&\E_{\a\sim\pi_1(s)}\Mp{V_{2,i}^\pi(s_2)-\underline{V}_{2,i}^\pi(s_2)+2C_{\P}\Norm{\psi(s_1,\a)}_{[\Sigma_{\D,1}^{\P}+\lambda I]^{-1}}+2C_{r}\Norm{\overline{\psi}_1(s_1,\a)}_{[\Sigma_{\D}^{r}+\lambda I]^{-1}}}\\
        \leq&\cdots\\
        \leq&\E_{\pi}\Mp{2C_{\P}\sum_{h=1}^H\Norm{\psi(s_h,\a_h)}_{[\Sigma_{\D,h}^{\P}+\lambda I]^{-1}}+2C_{r}\sum_{h=1}^H\Norm{\overline{\psi}_h(s_h,\a_h)}_{[\Sigma_{\D}^{r}+\lambda I]^{-1}}}.
    \end{align*}
    Similarly, for optimistic value estimation, we have
    \begin{align*}
        &\overline{V}_{1,i}^{\dagger,\pi_{-i}}(s_1)-V_{1,i}^{\dagger,\pi_{-i}}(s_1)\\
        =&\max_{a_i\in\A_i}\E_{\a_{-i}\sim\pi_{1,-i}(s_1)}\overline{Q}_{1,i}(s_1,\a)-\max_{a_i\in\A_i}\E_{\a_{-i}\sim\pi_{1,-i}(s_1)}Q^{\dagger,\pi_{-i}}_{1,i}(s_1,\a)\\
        \leq& \E_{\a\sim (\pi_i^\dagger,\pi_{-i})}\Mp{\overline{Q}_{1,i}(s_1,\a)-Q^{\dagger,\pi_{-i}}_{1,i}(s_1,\a)}\\
        \leq & \E_{\a\sim (\pi_i^\dagger,\pi_{-i})}\Mp{\B_1\overline{V}^{\dagger,\pi_{-i}}_{2}(s_1,\a)-\B_1V^{\dagger,\pi_{-i}}_{2}(s_1,\a)+2C_{\P}\Norm{\psi(s_1,\a_1)}_{[\Sigma_{\D}^{\P}+\lambda I]^{-1}}+2C_{r}\Norm{\overline{\psi}_1(s_1,\a)}_{[\Sigma_{\D}^{r}+\lambda I]^{-1}}}\\
        \leq & \cdots\\
        \leq & \E_{\pi_i^{\dagger},\pi_{-i}}\Mp{\sum_{h=1}^H2C_{\P}\Norm{\psi(s_h,\a_h)}_{[\Sigma_{\D}^{\P}+\lambda I]^{-1}}+2C_{r}\sum_{h=1}^H\Norm{\overline{\psi}_h(s_h,\a_h)}_{[\Sigma_{\D}^{r}+\lambda I]^{-1}}}.
    \end{align*}
\end{proof}

\begin{lemma}\label{lemma:gap bound}
    Under the event in Lemma \ref{lemma:value opt/pes},
    $$\text{Nash-gap}(\pi^\mathrm{output})\leq \sum_{i\in[m]}\Mp{\overline{V}_{1,i}^{\dagger,\pi^*_{-i}}(s_1)-\underline{V}_{1,i}^{\pi}(s)}.$$
\end{lemma}

\begin{proof}
    \begin{align*}
        \text{Nash-gap}(\pi^\mathrm{output})
        &= \max_{\pi'} \sum_{i\in[m]}\Mp{{V}_{1,i}^{\pi',\pi^*_{-i}}(s_1)-{V}_{1,i}^{\pi}(s)}\\
        &\leq \sum_{i\in[m]}\Mp{\overline{V}_{1,i}^{\dagger,\pi^*_{-i}}(s_1)-\underline{V}_{1,i}^{\pi}(s)}.
    \end{align*}
    Utilizing Lemma \ref{lemma:value opt/pes}, it is straightforward to derive this lemma. The proof is similar to the proof for Lemma 21 in \citep{cui2022provably}.
\end{proof}

\begin{theorem}
\label{thm:unilateral}
Set $\lambda=1$ for Algorithm \ref{algo:surrogate minimization}. With probability $1-\delta$, we have
    $$\text{Nash-gap}(\pi^{\mathrm{output}})\leq \max_{\pi_i}4\E_{\pi_i,\pi^*_{-i}}\Mp{\sum_{h=1}^HC_{\P}\Norm{\psi(s_h,\a_h)}_{[\Sigma_{\D,h}^{\P}+\lambda I]^{-1}}+C_{r}\sum_{h=1}^H\Norm{\overline{\psi}(s_h,\a_h)}_{[\Sigma_{\D}^{r}+\lambda I]^{-1}}},$$
    where $C_r=\widetilde{O}(\sqrt{dH})$ and $C_{\P}=\widetilde{O}(dH)$.
\end{theorem}

\begin{proof}
    By Lemma \ref{lemma:value bound} and Lemma \ref{lemma:gap bound}, we have
    \begin{align*}
        &\text{Nash-gap}(\pi^\mathrm{output})\\
        \leq& \sum_{i\in[m]}\Mp{\overline{V}_{1,i}^{\dagger,\pi^*_{-i}}(s_1)-\underline{V}_{1,i}^{\pi^*}(s)}\\
        \leq& \sum_{i\in[m]}\E_{\pi}\Mp{2C_{\P}\sum_{h=1}^H\Norm{\psi(s_h,\a_h)}_{[\Sigma_{\D,h}^{\P}+\lambda I]^{-1}}+2C_{r}\sum_{h=1}^H\Norm{\overline{\psi}_h(s_h,\a_h)}_{[\Sigma_{\D}^{r}+\lambda I]^{-1}}}\\
        &+\sum_{i\in[m]}\E_{\pi_i^{\dagger},\pi_{-i}}\Mp{\sum_{h=1}^H2C_{\P}\Norm{\psi(s_h,\a_h)}_{[\Sigma_{\D}^{\P}+\lambda I]^{-1}}+2C_{r}\sum_{h=1}^H\Norm{\overline{\psi}_h(s_h,\a_h)}_{[\Sigma_{\D}^{r}+\lambda I]^{-1}}}\\
        &+\sum_{i\in[m]}\Mp{V_{1,i}^{\dagger,\pi^*_{-i}}(s_1)-V_{1,i}^{\pi^*}(s_1)}\\
        \leq& \max_{\pi_i}4\E_{\pi_i,\pi^*_{-i}}\Mp{\sum_{h=1}^HC_{\P}\Norm{\psi(s_h,\a_h)}_{[\Sigma_{\D,h}^{\P}+\lambda I]^{-1}}+C_{r}\sum_{h=1}^H\Norm{\overline{\psi}_h(s_h,\a_h)}_{[\Sigma_{\D}^{r}+\lambda I]^{-1}}},
    \end{align*}
    where we leverage the fact that $V_{1,i}^{\dagger,\pi^*_{-i}}(s_1)-V_{1,i}^{\pi^*}(s_1)$ for Nash equilibrium $\pi^*$.
\end{proof}

Intuitively, the proof consists of two main phases: 1) we first reduce the MARL problem to the MARLHF problem, as preference signals can be sampled given the real rewards; 2) we then observe that in MARL problems, a Nash equilibrium is only identifiable when all adjacent actions are represented in the dataset.
This observation establishes the necessity of unilateral coverage. The sufficiency of unilateral coverage in MARLHF (Theorems \ref{thm: Unilateral Policy Coverage}, \ref{thm:unilateral}) is then derived from its sufficiency in MARL and the reduction from MARL to MARLHF.

\subsection{Auxiliary Lemmas}

\begin{lemma}\label{lemma:reward concentration} (Lemma 3.1 in \citep{zhu2023principled})
    With probability at least $1-\delta$, we have
    $$\Norm{\widehat{\theta}-\theta}_{\Sigma_{\D}^r+\lambda I}\leq C\sqrt{\frac{d+\log(1/\delta)}{\lambda^2}+\lambda B^2}.$$
\end{lemma}

\begin{lemma}\label{lemma:opt/pes}
    (Lemma A.1 in \citep{zhong2022pessimistic}) With probability at least $1-\delta$, we have
    $$0\leq \B_h\underline{V}_{h+1,i}(\cdot,\cdot)-\underline{Q}_{h,i}(\cdot,\cdot)\leq 2C_{\P}\Norm{\psi(\cdot,\cdot)}_{[\Sigma_{\D,h}^{\P}+\lambda I]^{-1}},$$
    $$0\geq \B_h\overline{V}_{h+1,i}(\cdot,\cdot)-\overline{Q}_{h,i}(\cdot,\cdot)\geq -2C_{\P}\Norm{\psi(\cdot,\cdot)}_{[\Sigma_{\D,h}^{\P}+\lambda I]^{-1}},$$
    where $C_{\P}=CdH\sqrt{\log(2dNH/\delta)}$.
\end{lemma}

%% file: text/Appendix_Experiments.tex
\section{Experiment Details}
\label{apd:experiment details}

\begin{algorithm}[t]
    \caption{Pipeline of Preference-Based Multi-agent Reinforcement Learning}
    \label{alg:experiment}
    \begin{algorithmic}[1]
        \State \textbf{Input:} Dataset $\mathcal{D}=\{\tau_{i}, \tau_{i}^{\prime}, \mathbf{y}_i\}_{i=1}^N$. 
        \State Train a agent-wise reward model $r_\phi$;
        \State Train a imitation model $\pi_{b}$;
        \State Apply MARL algorithm to learn $\pi_\mathbf{w}$ 
        with distribution-based pessimism from $\pi_{b}$;
        \State \Return $\pi_\mathbf{w}$.
    \end{algorithmic}
\end{algorithm}

\subsection{Implementation Details}

\paragraph{Pipeline}
Our pipeline for Preference-Based Multi-Agent Reinforcement Learning (PbMARL), detailed in Algorithm~\ref{alg:experiment}, outlines the key steps for training agents using preference datasets. The process begins with training an agent-specific reward model \(r_\phi\), followed by learning a imitation model \(\pi_b\). The final policy \(\pi_\mathbf{w}\) is then optimized using a MARL algorithm with distribution-based pessimism derived from \(\pi_b\). 

\paragraph{Model Configurations}
The models used in our experiments are designed to effectively handle the complexities of multi-agent environments. 
Our \textbf{reward model} $r_\phi$ is a fully connected neural network, featuring action and observation embedding layers followed by hidden layers.
The \textbf{MADPO agent network} uses RNN to output its Q-values, enabling the agent to make informed decisions based on its observations and learned policies. 
MABCQ and MAIQL, as tested in Section ~\ref{sec: Experiments}, are modified versions of BCQ \cite{fujimoto2019offpolicy} and IQL \cite{kostrikov2021offlinereinforcementlearningimplicit} tailored for MARL. Similar to VDN, the Q-functions in MABCQ and MAIQL are designed to represent the cumulative rewards of all agents collectively. This adaptation ensures compatibility with the multi-agent setting while preserving the theoretical and practical foundations of the original algorithms.
In MAIQL, the coefficient for expectile regression, $\tau$, is consistently set to 0.95 across all experiments. For MABCQ, random noise addition is omitted due to the discrete action spaces in the tested environments, and the VAE is replaced with a policy generator trained via imitation learning.
Table \ref{tab: hyperparameters} lists the main hyperparameters used in our experiments, while other details can be checked in our codebase: \url{https://github.com/NataliaZhang/marlhf}.

\begin{table*}[ht!]
\centering
\begin{tabular}{lr}
\toprule
\textbf{Hyperparameter} & \textbf{Default Value} \\
\midrule
MSE Loss Coefficient $\alpha$ & 1 (MPE), 1e-3 (Overcooked)\\
Pessimism Coefficient $\beta$ & 1 (Spread, Reference), 10 (Tag, Overcooked)\\
Prediction Steepness & 5 \\
Episode length & 26 (MPE), 400 (Overcooked) \\
Reward Model Type & MLP (Spread, Reference, Overcooked), RNN (Tag) \\
RNN Hidden Size & 64 (Tag) \\
MLP Layer Dimension & 64 \\
Reward Model Layers & 4 \\
Reward Model Epochs & 100 \\
Reward Model Learning Rate & 1e-3 \\
Reward Model Batch Size & 256 \\
IL Learning Rate & 1e-3 \\
IL Epochs & 100 \\
IL Batch Size & 128 \\
Policy Model Learning Rate & 1e-3 \\
Policy Model Epochs & 1e4 (Spread, Reference), 1e6 (Tag), 1e5 (Overcooked) \\
Policy Model Batch Size & 128 \\
MAIQL Expectile Regression Coefficient & 0.95 \\
Optimizer & Adam \\
\bottomrule
\end{tabular}
\caption{\footnotesize{Main hyperparameters in experiments}}
\label{tab: hyperparameters}
\end{table*}

\paragraph{Dataset Configurations}
\label{apd: dataset}
As mentioned in Section \ref{subsec: datasets}, we collected trajectories in different environments using various policies to ensure a diverse dataset. 
Each dataset contains 38400 trajectories in each MPE environment and 960 trajectories in Overcooked.
The number of trajectory pairs is chosen as 10 times the number of trajectories.
Preference tags were then generated for these trajectory pairs in the mixed datasets. 
To adjust the randomness of the preferences, a steepness parameter was introduced as a scalar of the standardized reward. This configuration ensures a comprehensive dataset that can effectively support the evaluation of our methods.

\subsection{Tasks Descriptions}
\label{apd: Tasks Descriptions}
MPE is chosen for our experiments due to its versatility and well-established use as a benchmark for MARL algorithms. Among its variety of scenarios, the following three methods are chosen for our experiments:
\begin{itemize}
    \item Simple Spread
    \begin{itemize}
        \item Objective: Group of agents spread out. Each agent aims to occupy a unique landmark while avoiding collisions with other agents.
        \item Challenge: There is a potential conflict between the collision penalty and the spreading goal. Any biased policy would pushes the agents away from their targets, leading to suboptimal performance. Successfully balancing these objectives is critical to avoid negative learning outcomes.
    \end{itemize}
    \item Simple Tag
    \begin{itemize}
        \item Objective: Adversaries aim to catch the good agents cooperatively, while good agents aim to avoid being caught.
        \item Challenge: The adversaries only gets reward at the timestep of catching a good agent, so recovering the reward distribution across time becomes a challenging work. Note that the original environment is a 1v3 adversary game, and we convert it into a 3-agent cooperative game for better evaluation by fixing the good agent with a MAPPO pretrained policy. This environments requires high operation precision.
    \end{itemize}
    \item Simple Reference
    \begin{itemize}
        \item Objective: Agents aim to reach target landmarks that are known only to others by communication.
        \item Challenge: The requirement for communication increases the complexity of the action space and the dependency among cooperating agents. The performance of agents is affected particularly under unilateral policy conditions, where misaligned communication signals can significantly impact performance. 
    \end{itemize}
    \item Overcooked
    \begin{itemize}
        \item Objective: Two agents in a gridworld scoring points by repeatedly completing a three-step process: gathering ingredients, cooking, and serving dishes. Each step requires interacting with the environment at specific locations and orientations.
        \item Challenge: The episode length in this environment is notably long (400 timesteps), with sparse reward signals. This creates significant challenges for training the reward model, as it may incorrectly attribute rewards to specific behaviors, resulting in inaccurate training objectives and suboptimal learning outcomes.
    \end{itemize}
\end{itemize}

These tasks provide a robust framework for evaluating the effectiveness and adaptability of our offline MARLHF algorithm in various multi-agent settings. Additionally, they represents the common environments that are sensitive to dataset coverage, where dataset with unilateral policy can easily disrupt cooperation. Therefore, robust approaches are essential to ensure stable performance across different scenarios.

\subsection{Scalability Analysis}
\label{apd: scalability analysis}

To evaluate the scalability of our approach, we tested the performance of different methods as the number of agents increased. The experiments were conducted in the Spread-v3 environment, and the test returns per agent were recorded in Table \ref{tab:scale-spread}.

In our experiments, we observed that as the number of agents increases, convergence times lengthen and the complexity of the problem grows, mirroring the challenges typically encountered in traditional MARL settings. While our current approach manages this scaling without introducing new problems, it does not specifically address the inherent issues of instability and complexity that are well-documented in traditional MARL.

Further work may involve optimizing the algorithms to better handle larger-scale multi-agent environments or exploring alternative methods that maintain high performance even as the agent count increases.

\begin{table*}[h]
\centering
\begin{tabular}{lllll}
\toprule
 & 4 agents & 5 agents & 6 agents & 7 agents \\
\midrule
Mix-Unilateral & -31.13 \std{0.33} & -28.26 \std{0.43} & -26.92 \std{0.33}& -25.48 \std{0.13} \\
Pure-Expert & -31.71 \std{0.17} & -28.80 \std{0.10} & -27.16 \std{0.39} & -26.29 \std{0.32} \\
Trivial & -50.83 & -36.92 & -28.56 & -23.62 \\
\bottomrule
\end{tabular}
\caption{Test returns per agent in spread-v3 when agent scales. We ran 5 seeds for each dataset and kept all parameters at their default values ($\alpha=1, \beta=1$). Trivial represents test returns where all agents take a random policy, serving as a comparison. As the number of agents scales, the performance of the method generally decrease, and is eventually outperformed by the trivial policy when it reaches 7 agents.}
\label{tab:scale-spread}
\end{table*}

\subsection{Ablation Study Details}
\label{apd: ablation study}
\begin{figure}[ht!]
    \centering
    \includegraphics[width=.6\textwidth]{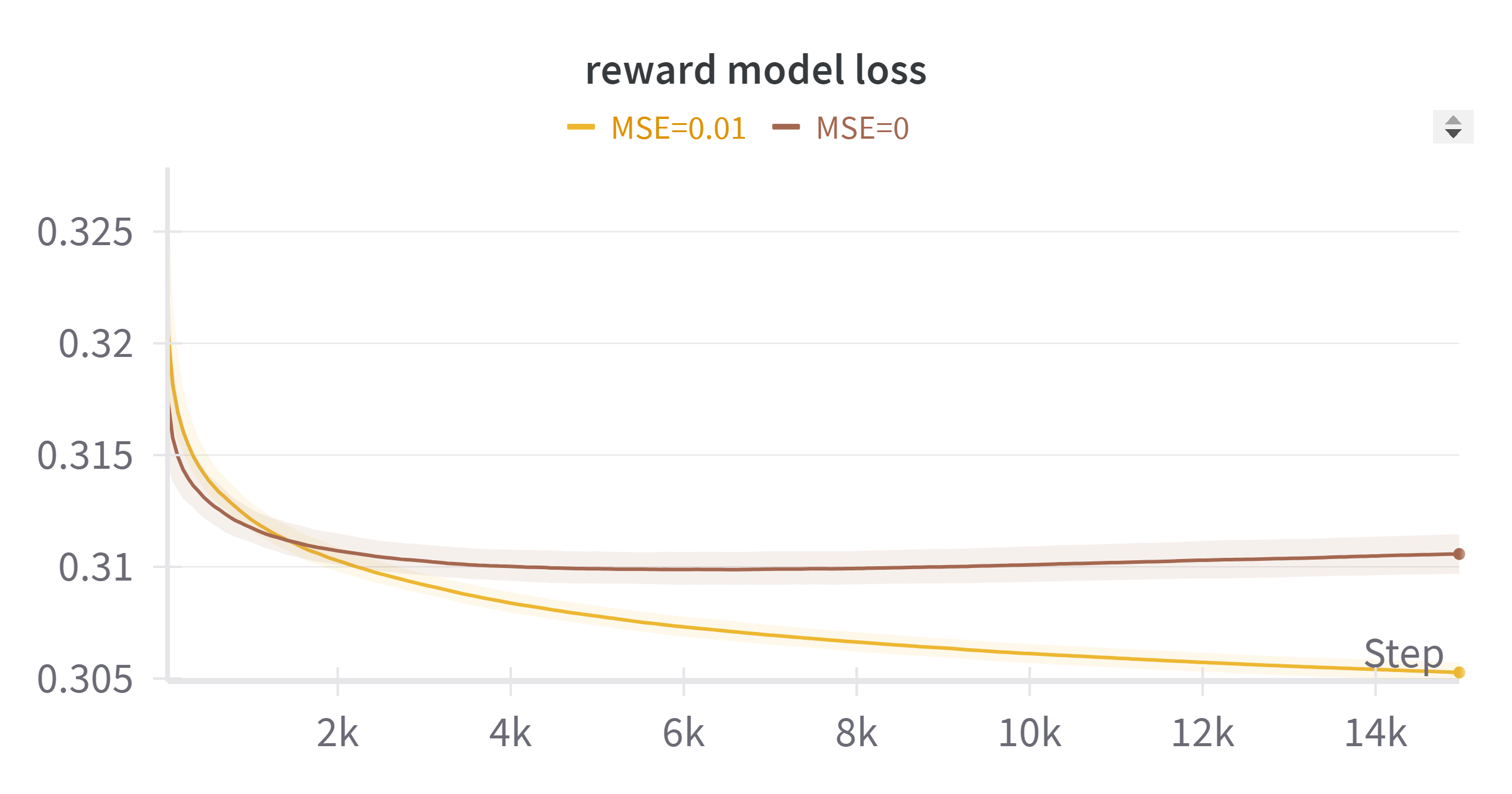}
    \caption{Reward model training curves on Spread-v3 Diversified dataset. Extra positive MSE regularization results in lower final training loss.}
    \label{fig:RM-mse}
\end{figure}

\begin{figure}[H]
    \centering
    \begin{subfigure}[b]{0.3\textwidth}
        \includegraphics[width=\textwidth]{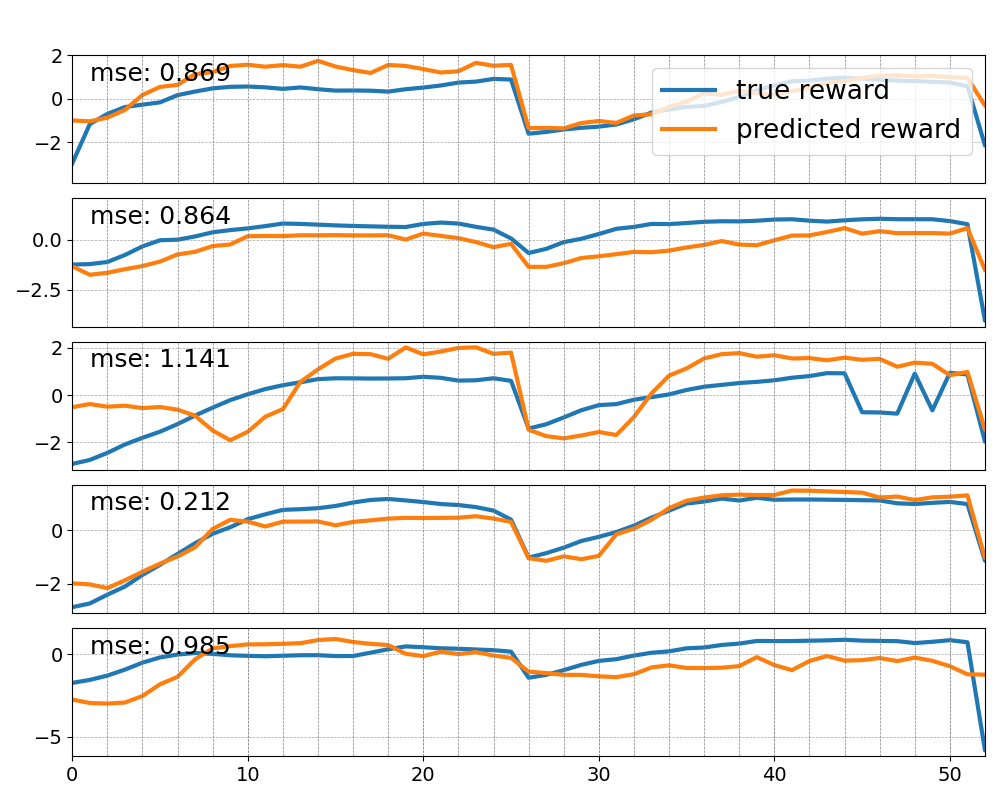}
        \caption{\parbox{0.9\linewidth}{\centering Pure-Expert \\ Spread}}
        \label{fig:rewardlines_spread_pure_001}
    \end{subfigure}
    \hfill
    \begin{subfigure}[b]{0.3\textwidth}
        \includegraphics[width=\textwidth]{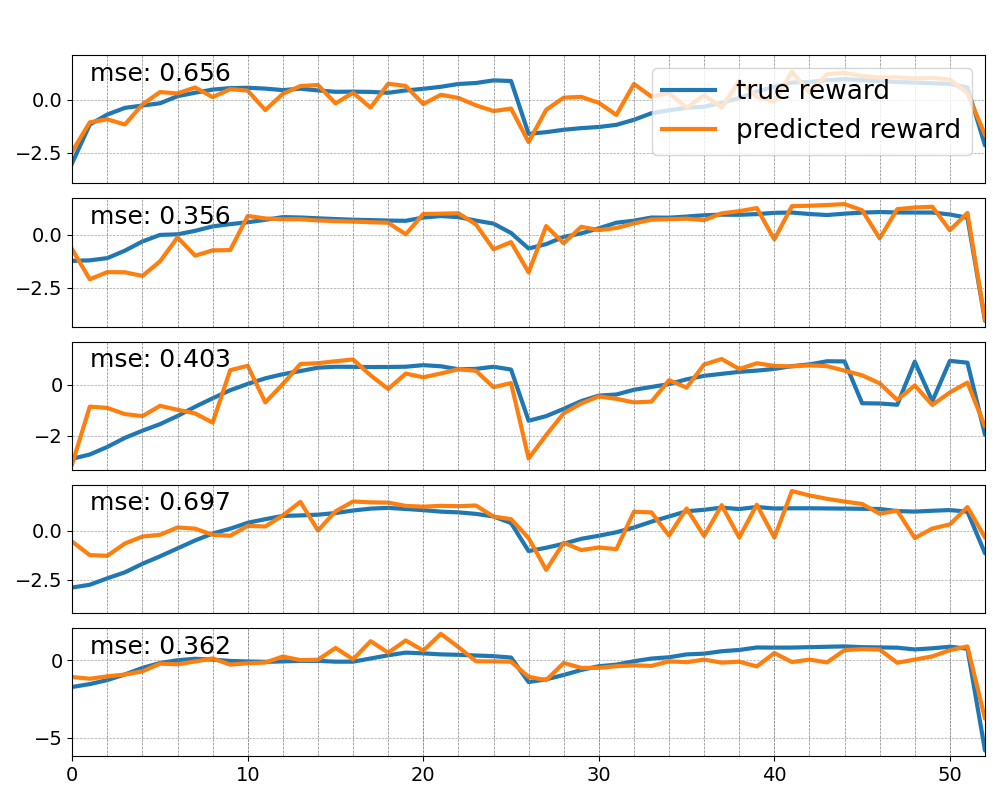}
        \caption{\parbox{0.9\linewidth}{\centering Diversified $\alpha=0$ \\ Spread}}
        \label{fig:rewardlines_spread_diverse_0}
    \end{subfigure}
    \hfill
    \begin{subfigure}[b]{0.3\textwidth}
        \includegraphics[width=\textwidth]{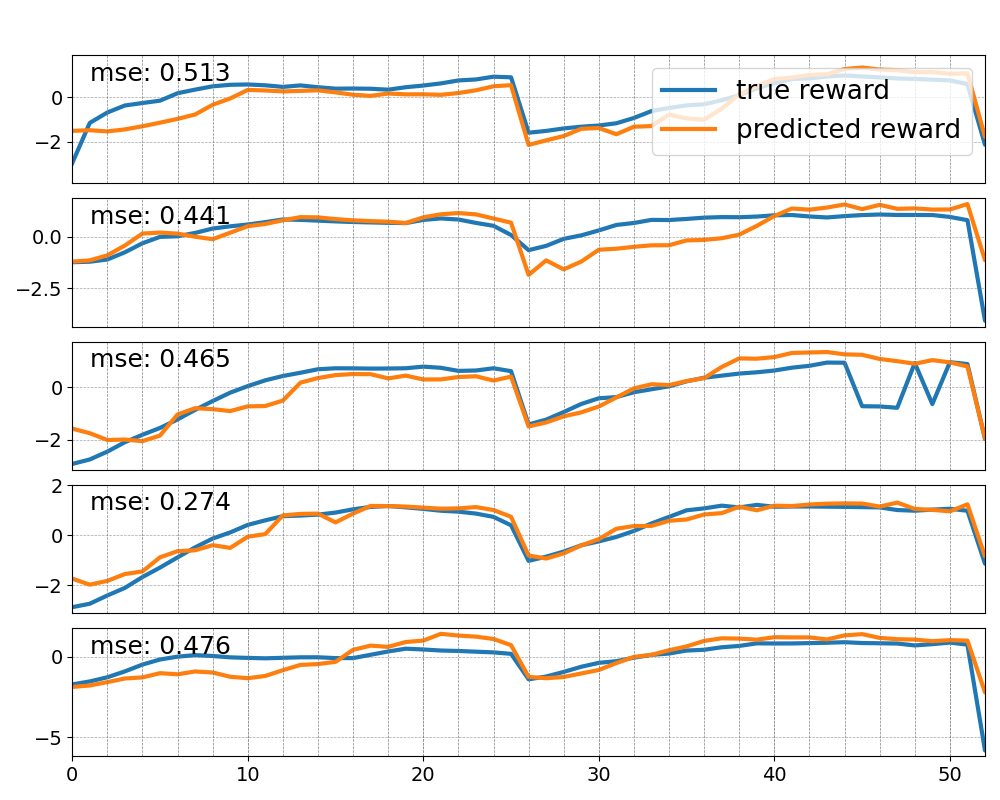}
        \caption{\parbox{0.9\linewidth}{\centering Diversified \\ Spread}}
        \label{fig:rewardlines_spread_diverse_001}
    \end{subfigure}
    \begin{subfigure}[b]{0.3\textwidth}
        \includegraphics[width=\textwidth]{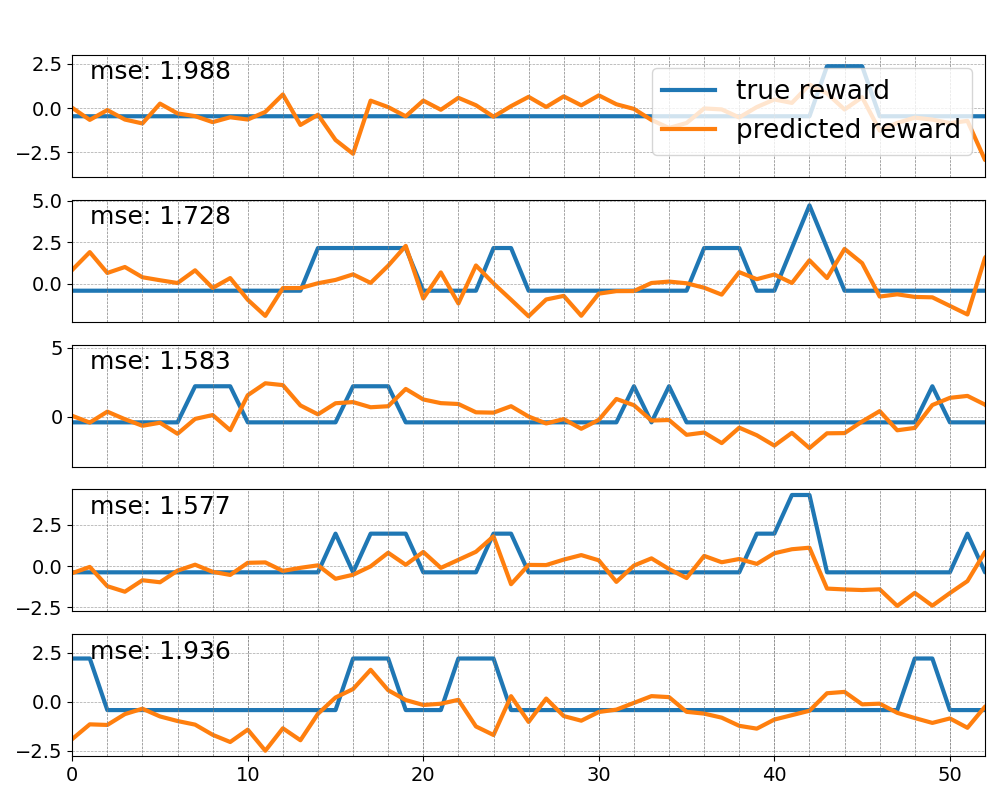}
        \caption{\parbox{0.9\linewidth}{\centering Pure-Expert \\ Tag}}
        \label{fig:rewardlines_tag_pure_0001}
    \end{subfigure}
    \hfill
    \begin{subfigure}[b]{0.3\textwidth}
        \includegraphics[width=\textwidth]{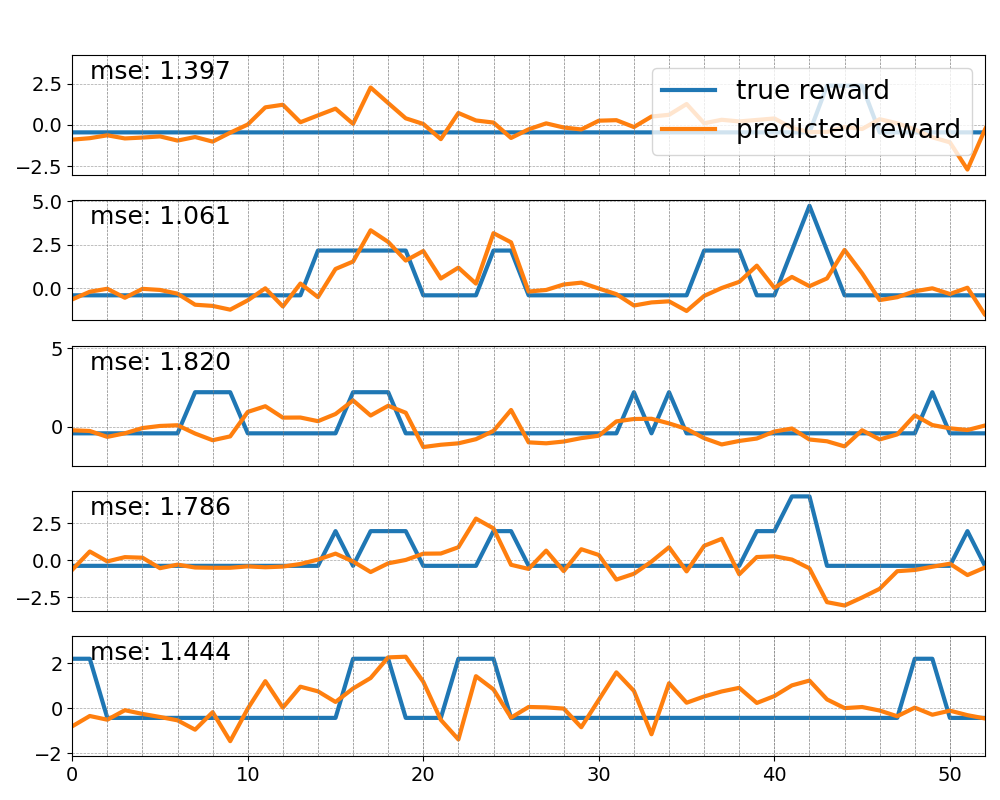}
        \caption{\parbox{0.9\linewidth}{\centering Diversified $\alpha=0$ \\ Tag}}
        \label{fig:rewardlines_tag_diverse_0}
    \end{subfigure}
    \hfill
    \begin{subfigure}[b]{0.3\textwidth}
        \includegraphics[width=\textwidth]{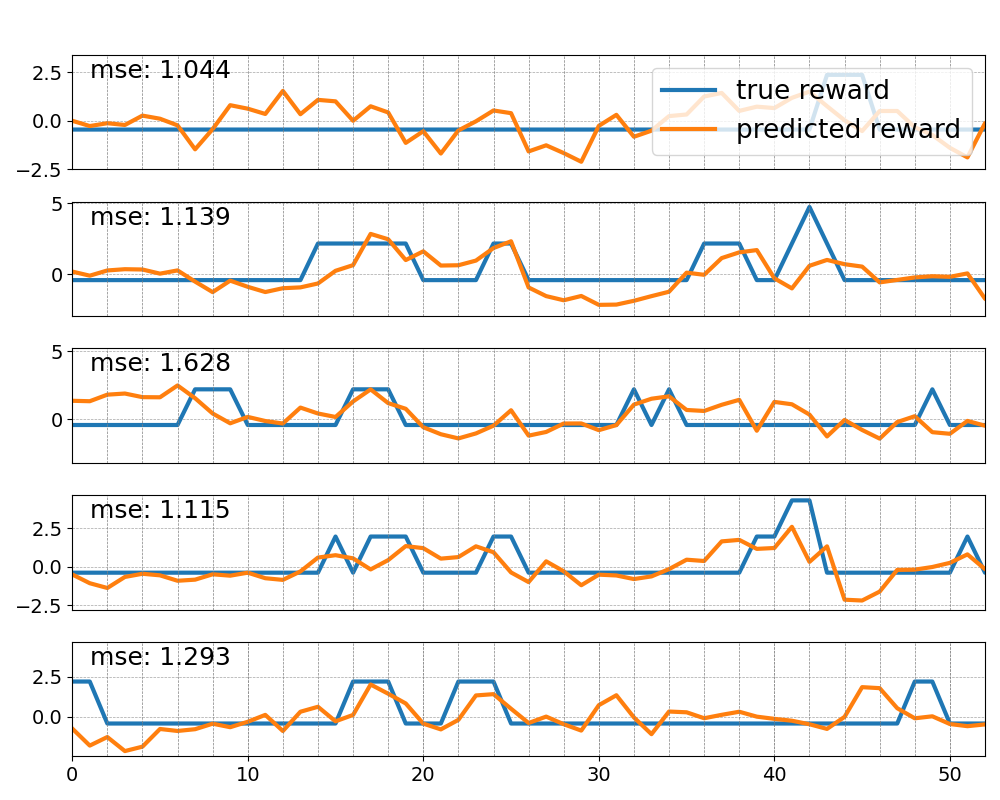}
        \caption{\parbox{0.9\linewidth}{\centering Diversified \\ Tag}}
        \label{fig:rewardlines_tag_diverse_0001}
    \end{subfigure}
    \begin{subfigure}[b]{0.3\textwidth}
        \includegraphics[width=\textwidth]{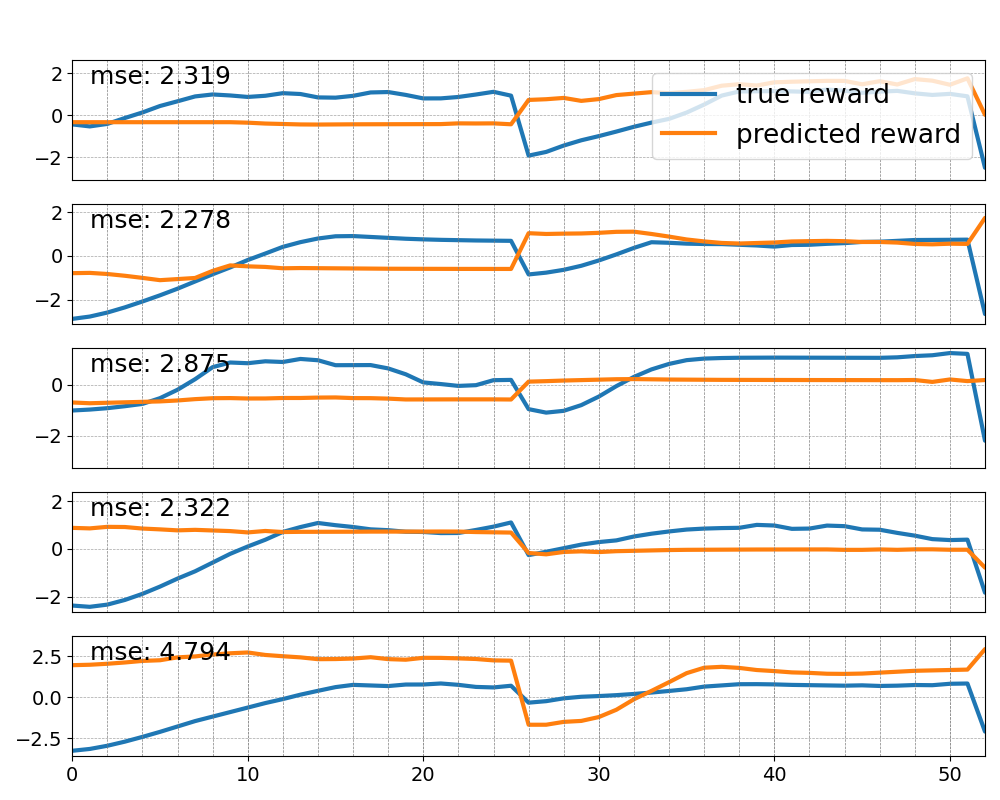}
        \caption{\parbox{0.9\linewidth}{\centering Pure-Expert \\ Reference}}
        \label{fig:rewardlines_reference_pure_1000}
    \end{subfigure}
    \hfill
    \begin{subfigure}[b]{0.3\textwidth}
        \includegraphics[width=\textwidth]{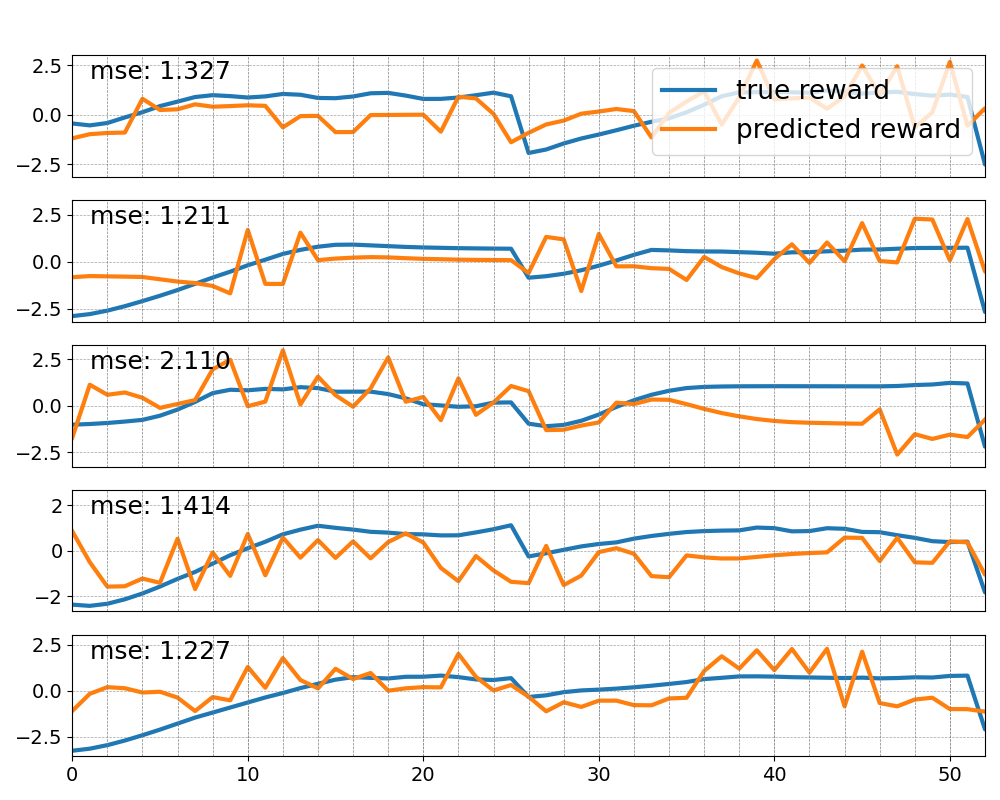}
        \caption{\parbox{0.9\linewidth}{\centering Diversified $\alpha=0$ \\ Reference}}
        \label{fig:rewardlines_reference_diverse_0}
    \end{subfigure}
    \hfill
    \begin{subfigure}[b]{0.3\textwidth}
        \includegraphics[width=\textwidth]{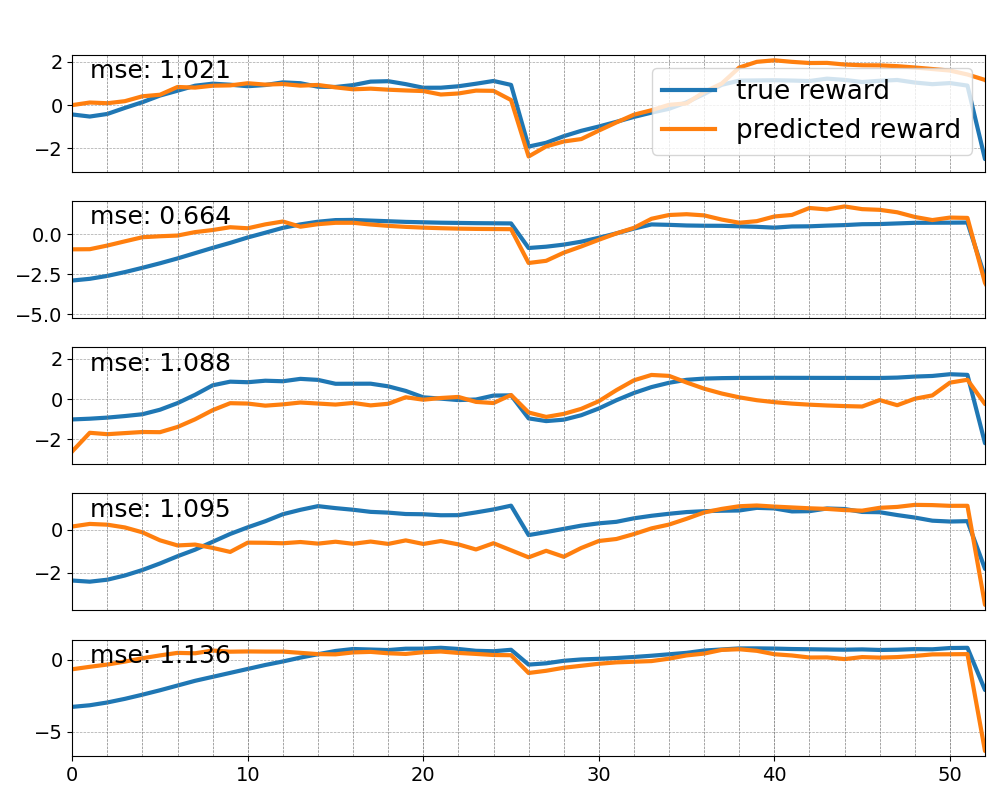}
        \caption{\parbox{0.9\linewidth}{\centering Diversified \\ Reference}}
        \label{fig:rewardlines_reference_diverse_1000}
    \end{subfigure}
    
    \caption{Predicted rewards and ground truth (both standardized) in all environments. Our method with diversified dataset and reward regularization gives predictions that approximate the ground truth the best. }
    \label{fig:rewardlines_appendix}
\end{figure}

\begin{table*}[ht!]
\centering
\begin{tabular}{lrrrrrrrrr}
\toprule
$\alpha$ & 0 & 0.001 & 0.01  & 0.1  & 1 & 10 & 100 & 1000 \\
\midrule
Spread-v3 & 0.350 & 0.345 & 0.347 & 0.351 & 0.361 & 0.389 & 0.460 & 0.603 \\
Tag-v3 & 0.465 & 0.431 & 0.440 & 0.455 & 0.484 &0.531 & 0.603 & 0.676 \\
Reference-v3 & 0.358 & 0.356 & 0.362 & 0.374 & 0.393 & 0.434 & 0.508 & 0.623 \\
\bottomrule
\end{tabular}
\caption{\footnotesize{NLL loss over diversified dataset. Appropriate regularization can assist the reward model in learning more effectively, leading to a reduction in NLL loss. Strengthening regularization (larger $\alpha$) sometimes leads to lower NLL loss, indicating better capability of capturing differences in reward signal.}}
\label{tab: mse-ablation}
\end{table*}

\begin{table*}[t]
\centering
\begin{tabular}{lrrrr}
\toprule
  & Spread-v3 & Tag-v3 & Reference-v3 & Overcooked\\
& MSE & MSE & MSE & MSE\\
\midrule
Diversified  & 0.434 & 1.46 & 1.19 & 2.04\\
Mix-Unilateral  & 0.647 & 1.52 & 1.09 & 1.98\\
Mix-Expert & 0.578 & 1.78 & 1.09 & 2.17 \\
Pure-Expert  & 0.673 & 1.48 & 2.33 & 1.72\\
\bottomrule
\end{tabular}
\caption{The mean squared error (MSE) between the standardized predicted rewards and the standardized ground truth rewards. 
\vspace{-0.5cm}
}
\label{tab:MSE}
\end{table*}

In this section, we explore the effects of specific components of our method, focusing on the influence of MSE regularization and the use of diversified datasets.
Our ablation studies collectively underscore the importance of MSE regularization and diversified datasets in enhancing the robustness and accuracy of the reward models within our framework.

The incorporation of MSE regularization plays an important role in improving model stability and convergence. As seen in Figure \ref{fig:RM-mse}, appropriate regularization leads to lower final training loss, suggesting a more stable learning process. 
Lower MSE between predicted reward and ground truth (significantly below 2) can indicate a strong correlation between the two (cf. ~\ref{tab:MSE}). In the MPE experiments, the reward model predictions align closely with the ground truth, and the optimization benefits from dataset diversity are particularly pronounced in the Reference-v3 and Spread-v3 scenarios. 
However, in more complex environments, the reward may exhibit more patterns, making MSE less effective as a metric for assessing reward model quality.
For example, in the Overcooked environment, assigning rewards for \textbf{cooking the dish} and \textbf{serving the dish} results in very similar returns, as a complete scoring cycle involves both actions, but these two reward function will have squared difference of 2.

Additionally, our evaluation of the Negative Log-Likelihood (NLL) loss over diversified datasets reveals that stronger regularization can sometimes lead to a reduction in NLL loss, as listed in Table \ref{tab: mse-ablation}. This implies that the model becomes more capable of capturing nuances in the reward signals as the regularization parameter, $\alpha$, is increased. However, it is also important to balance the strength of regularization, as overly strong regularization could potentially hinder the model’s flexibility in capturing complex reward structures.

The interplay between regularization strength and dataset diversity is critical for achieving optimal model performance in complex multi-agent settings.
